%% file: main.tex
\documentclass[twoside]{article}

\let\LaTeXaddcontentsline\addcontentsline 

\usepackage{float}   % for [H]

\newcommand{\cE}{\mathcal{E}}
\newcommand{\stay}{{\star}}
\newcommand{\E}{\mathbb{E}}
\newcommand{\PP}{\mathbb{P}}
\newcommand{\ind}{\mathbf{1}}
\newcommand{\cF}{\mathcal{F}}
\newcommand{\cV}{\mathcal{V}}
\newcommand{\cC}{\mathcal{C}}
\newcommand{\cS}{\mathcal{S}}
\newcommand{\cR}{\mathcal{R}}
\newcommand{\cL}{\mathcal{L}}

\newcommand{\sL}{\mathsf{L}}
\newcommand{\sR}{\mathsf{R}}

\newcommand{\Fh}{\widehat F}
\newcommand{\Fb}{\overline F}
\newcommand{\Ah}{\widehat A}
\newcommand{\Dt}{\widetilde D}
\newcommand{\Tt}{\widetilde T}
\newcommand{\Kt}{\widetilde K}
\newcommand{\cens}{\mathsf{C}}

\newcommand{\EE}{\mathbb{E}}
\newcommand{\RR}{\mathbb{R}}

\newcommand{\Fbar}{\bar{F}}
\newcommand{\Flo}{L}
\newcommand{\Fhi}{U}
\newcommand{\probefunction}{f}

\newcommand{\Fhat}{\widehat{F}}

\newcommand{\Es}{\EE_{s-1}}
\providecommand{\trsym}[1]{\mathsf{#1}}%
\usepackage[textsize=tiny]{todonotes}
\usepackage{amsthm,thmtools,thm-restate}
\usepackage[round]{natbib}
\usepackage{amsmath,amssymb,amsthm}
\usetikzlibrary{calc,arrows.meta}

\theoremstyle{plain}
\theoremstyle{definition}
\newtheorem{theorem}{Theorem}[section]
\newtheorem{proposition}[theorem]{Proposition}

\newtheorem{lemma}[theorem]{Lemma}
\newtheorem{corollary}[theorem]{Corollary}
\newtheorem{definition}[theorem]{Definition}
\newtheorem{assumption}[theorem]{Assumption}

\theoremstyle{remark}
\newtheorem{remark}[theorem]{Remark}

\newcommand{\hF}{\widehat{F}}
\newcommand{\bF}{\overline{F}}

\newcommand{\bS}{\overline{S}}
\newcommand{\tT}{\widetilde{T}}

\newcommand{\Ft}{\mathcal{F}}
\newcommand{\R}{\mathbb{R}}

\newcommand{\Cells}{\mathcal{C}}

\newcommand{\Alp}{\mathcal{K}}
\newcommand{\V}{\mathcal{V}}

\newcommand{\Lcode}{\ell}
\newcommand{\NL}{N^{\mathsf L}}
\newcommand{\NR}{N^{\mathsf R}}

\providecommand{\RR}{\mathbb R}

\providecommand{\PP}{\mathbb P}
\providecommand{\EE}{\mathbb E}
\providecommand{\V}{\mathcal V}                      % probe class
\providecommand{\hF}{\widehat F}                     % empirical CDF
\providecommand{\bF}{\overline F}                    % running-average conditional CDF
\providecommand{\bS}{\overline S}                    % running-average survival
\providecommand{\Dt}{\widetilde D}                   % envelope sum \Dt_t(a)
\providecommand{\Lcode}{\ell}                        % back-compat alias for \price

\usepackage[preprint]{aistats2027}
 \usepackage[hidelinks]{hyperref}
\begin{document}
\addtocontents{toc}{\protect\setcounter{tocdepth}{-1}}  % keep main-text sections out



\twocolumn[

\aistatstitle{Anytime Valid Confidence Bands for Distribution Functions}

\aistatsauthor{
    Bror Hultberg \And
    Dave Zachariah \And
    Ant\^onio H. Ribeiro
}

\aistatsaddress{
    Uppsala University, Sweden \And Uppsala University, Sweden \And Uppsala University \& \\ SciLifeLab, Sweden
}
]

\begin{abstract}
Nonparametric inference often requires characterizing an unknown function, such as a cumulative distribution function (CDF), across its entire domain rather than a specific point. We develop confidence bands for distribution functions of a data stream that are valid simultaneously over all times and all points, even when each observation may depend arbitrarily on the past. We begin by deriving such bands for the time-averaged conditional CDF and show how their widths are determined by the depth of a binary search tree. We then show that this result is an instance of a framework through which one can construct valid confidence bands uniformly over time and over an arbitrary index set for any indexed family of time-averaged distribution functions. We illustrate the framework on importance-weighted and right-censored CDFs, as well as multivariate CDFs.
\end{abstract}

\section{INTRODUCTION}

Consider a stream of observations $X_1,X_2,\ldots$ with no assumption that the observations are independent or identically distributed. In particular, the distribution of $X_i$ may depend arbitrarily on the observations $X_1,\ldots,X_{i-1}$, and may therefore change over time. This setting arises naturally in adaptive experiments, online decision making, and sequential monitoring, where the mechanism generating each observation may depend on the history of the process.

What statistical properties can be estimated from such a stream? For each $i$, consider the cumulative distribution function (CDF) of $X_i$,
\begin{equation}\label{eq:cond-cdf}
F_i(v)=\PP\left(X_i \leq v \mid X_1,\ldots,X_{i-1}\right),
\end{equation}
conditional on the preceding observations, which may change arbitrarily with $i$ and with the realized history. Since $F_i$ describes the distribution of a single future observation, there is generally not enough information to estimate any individual $F_i$. Instead, following \citet{mineiro2023}, we consider the average of the conditional distributions up to any time $t$,
\begin{equation}\label{eq:estimand}
\Fbar_t(v) =\frac{1}{t}\sum_{i=1}^{t} F_i(v).
\end{equation}
Unlike the classical population CDF, $\Fbar_t$ is itself a time-indexed target. As $t$ increases, another conditional distribution enters the average. In the i.i.d.\ case, $F_i=F$ for every $i$, so $\Fbar_t=F$ for all $t$, recovering the classical setting.

We seek confidence bands $(\Flo_t(v), \Fhi_t(v))$ for $\Fbar_t(v)$ that are \emph{simultaneously} valid over \emph{all} $v$ and over all times $t$.

This problem can be viewed more generally by considering an observable function $\probefunction(x,v)$ indexed by $v\in\mathcal V$, where $\mathcal V$ is an arbitrary space.
The conditional CDF in \eqref{eq:cond-cdf} corresponds to
\begin{equation}\label{eq:cond-mean}
F_i(v)=\EE\left[\probefunction(X_i,v)\mid X_1,\ldots,X_{i-1}\right],
\end{equation}
with
\[
\probefunction(x,v)=\ind\{x\leq v\},
\qquad v\in\RR,
\]
as the chosen family of functions. We refer to the argument $v$ as a \emph{probe} and consider the time-averaged conditional mean
\begin{equation}\label{eq:general-target}
\mu_t(v) =\frac{1}{t}\sum_{i=1}^{t}
\EE\left[\probefunction(X_i,v)\mid X_1,\ldots,X_{i-1}\right],
\quad v\in\mathcal V.
\end{equation}
Our more general goal is to construct confidence bands $(\Flo_t(v), \Fhi_t(v))$ for $\mu_t(v)$ that are simultaneously valid over both $t$ and $v$. This is an inferential problem for a wider family of distribution functions, including counterfactual, right-censored and multidimensional CDFs.

Our contributions are as follows:
\begin{itemize}

\item We introduce a framework for anytime-valid bounds that hold simultaneously over a family of
statistics indexed by an arbitrary set $\mathcal V$. For any bounded family and any predictable
refinement of $\mathcal V$, it yields a band whose width at an index is a discretization error over
the cells containing that index, plus the code length of their refinements, rather than a cost
growing with the size of $\mathcal V$.

    \item Specializing to threshold functions on $\R$, we obtain anytime-valid confidence bands for the CDF that do not require i.i.d.\ data, covering instead the running average of the conditional CDFs, with a cost of uniformity that is logarithmic in $t$ for i.i.d.\ continuous data.
    
    \item We apply the framework to derive bounds of importance-weighted distribution functions with possibly unbounded weights; right-censored functions with a known censoring law; and multidimensional CDFs.
\end{itemize}
We illustrate our results with numerical experiments.

\section{PROBLEM FORMULATION}

Our goal is to construct confidence bands $(\Flo_t,\Fhi_t)$ that infer distribution functions $\mu_t(v)$ in \eqref{eq:general-target} with uniform validity over $v$ and time $t$. That is, given a tolerated miscoverage level $\alpha$, the bands should satisfy:
\begin{equation}\label{eq:general-goal}
\boxed{\PP\big(
\forall t\geq 1,\ \forall v\in\mathcal V:
\Flo_t(v)\leq\mu_t(v)\leq\Fhi_t(v)
\big)
\geq 1-\alpha.}
\end{equation}
We form the bands using the running empirical average:
\[
\widehat\mu_t(v)
:=
\frac{1}{t}\sum_{i=1}^t \probefunction{}(X_i,v),
\]
obtained from the stream $X_1, X_2, \dots, X_t$.
For every \emph{fixed} $v$, it is possible to derive a sequence of anytime-valid confidence bands for $\mu_t(v)$, using time-uniform martingale concentration inequalities \citep{chernoff}. The challenge we face is to establish \emph{uniformity} simultaneously with respect to $v$ in \eqref{eq:general-goal}, with no assumptions about the dependence between samples.

Without additional assumptions, such uniform convergence can fail even asymptotically. \citet{Rakhlin2015} characterize this phenomenon through the theory of sequential uniform convergence: under arbitrary dependence, a function class may fail to satisfy a uniform law of large numbers for martingale differences. For the class of threshold functions, this means that the empirical CDF may not converge uniformly to the running-average CDF in \eqref{eq:estimand}. \\

\section{RELATED WORK}
\textbf{Anytime-valid inference.}
Anytime-valid inference based on nonnegative supermartingales dates back to \citet{ville1939} and \citet{10.1214/aoms/1177696786}. See \citet{Savisurvey} for a recent survey. More modern developments include \citet{10.1214/20-AOS1991}, which provides general tools for constructing anytime-valid confidence sequences and related inferential procedures.

\paragraph{Uniform confidence bands for CDFs.}
For i.i.d.\ data, the classical Dvoretzky--Kiefer--Wolfowitz inequality
\citep{dkw,massart} gives fixed-time,
value-uniform confidence bands for CDFs, \citet{10.3150/21-BEJ1388} develop
bands that are uniform over time and quantile level, and
\citet{CLERICO2026110558} use coin betting to obtain variance-adaptive bands
for the mean of a monotone process, such as the CDF, that are uniform over time
and over its argument.

\textbf{Off-policy evaluation.}
For off-policy evaluation, several works study confidence bands for the target CDF at a fixed time, including \citet{chandak2021universal} and \citet{huang2021offpolicy}. \citet{anytimeoffpolicy} develops a time-uniform analogue. 

\textbf{Nonstationarity.} The CDF and off-policy works above assume that the target distribution is fixed.
Several prior works have considered estimands based on temporal averages. \citet{anytimeoffpolicy} develop confidence sequences (CSs) for a time-varying average policy value, while \citet{Ham_2026} study the temporal average of the contemporaneous causal effect.
The work most closely related to ours is \citet{mineiro2023}, who consider the time-varying estimand in \eqref{eq:estimand}. They provide the first time- and value-uniform confidence bands for this estimand, obtained via a fixed dyadic multiresolution union bound, with a width guarantee that adapts to the smoothness of the underlying estimand. We adopt their estimand and their ``always valid and sometimes trivial'' design
principle, and differ in the mechanism (data-driven rather than fixed
discretization).
In particular, our construction yields CDF bands that are invariant under any strictly increasing transformation of the data that is fixed in time. The transformation need not be known, since the cells depend only on the relative order of the probe and the observations. \citet[Section~6]{mineiro2023} speculate that bands covariant to a fixed but unknown monotonic transformation are possible without violating known impossibility results, and our construction provides such bands.

\textbf{Sequential uniform convergence.} Some triviality is unavoidable for time-varying estimands. For dependent data, uniform convergence of empirical averages to their
conditional means is governed by sequential analogues of classical complexity
measures. For classes of indicator functions, such as thresholds, the relevant
one is the Littlestone dimension, the sequential analogue of the VC dimension \citep{Rakhlin2015,alon}. Thresholds on $\R$ have infinite Littlestone
dimension, so no uniform law holds for the
running average of conditional CDFs, and an adaptively chosen sequence can keep
the empirical CDF away from the estimand at some threshold for all time.

\section{RESULTS FOR CDF}\label{sec:cdf}
We first consider the construction of a confidence band for the time-averaged CDF \eqref{eq:estimand}. The key idea is to construct the band using a partition of the space $\mathcal{V} = \mathbb{R}$, which is refined as observations arrive. This enables bands that adapt their  resolution to regions where the data provide information, while maintaining uniform coverage over both time and the value of the CDF.\\
To construct this partition, consider the stream of observations $X_1,X_2,\ldots$, which  recursively partition the real line into intervals: each new observation splits the interval in which it falls into two subintervals. We formalize this notion as follows.

\begin{definition}\label{def:cellsR}
For a probe $v\in\mathbb R$ define its \emph{cell} at time $s$ as
\begin{equation}\label{eq:cellsR}
  C_s(v):=\bigl[l_s(v),\,r_s(v)\bigr),
\end{equation}
where $l_s(v):=\max\{X_r:r<s,X_r\le v\}$ and
$r_s(v):=\min\{X_r:r<s,X_r>v\}$ use all samples up to time
$s-1$.\footnote{With the convention that $\max\emptyset:=-\infty$,
$\min\emptyset:=+\infty$ and $[-\infty,r):=(-\infty,r)$.}
We say that the cell of $v$ is \emph{split} at time $s$ if $X_s$ falls \emph{strictly}
inside $C_s(v)$, that is, $l_s(v)<X_s<r_s(v)$. The split is a
\emph{right split} if $X_s\le v$ and a \emph{left split} if $X_s>v$. Let
\begin{align*}
  \NR_t(v)&:=\#\{s\le t:\ \text{right split of $v$ at $s$}\},\\
  \NL_t(v)&:=\#\{s\le t:\ \text{left split of $v$ at $s$}\},
\end{align*}
and let $N_t(v):=\NR_t(v)+\NL_t(v)$ be the number of times $v$ is split up to
time $t$.
\end{definition}

After a split, the cell of $v$ is $[l_s(v),X_s)$ for a left split and
$[X_s,r_s(v))$ for a right split. Otherwise it is unchanged.
\subsection{Confidence bands for CDF}

\begin{restatable}{theorem}{cdfbandtrivial}
\label{thm:cdfbandtrivial}
Let the miscoverage level be $\alpha\in(0,1)$ and set $\lambda > 0$. Then the
confidence band given by
\begin{align}
\Flo_t(v) &:= \hF_t(v)-\frac{\NR_t(v)}{t}-\rho_t\bigl(N_t(v)\bigr),
\label{eq:lower}\\
\Fhi_t(v) &:= \hF_t(v)+\frac{\NL_t(v)}{t}+\rho_t\bigl(N_t(v)\bigr),
\label{eq:upper}
\end{align}
where
\[
\rho_t(n)\;:=\;\frac{\lambda}{8}+\frac{n\ln2+\ln\frac2\alpha}{\lambda t},
\]
is valid uniformly over time and $v$, i.e., it satisfies
\eqref{eq:general-goal} with $\mu_t=\Fbar_t$.
\end{restatable}

\begin{figure}
    \centering
    \input{fig_tree_column}
    \caption{Example in which we observe the sequence $\{X_1,X_2,X_3\}=\{3,1,5\}$
    and use \eqref{eq:cellsR} to successively construct cells $C_s(v)$ in
    $\mathcal{V}=\mathbb{R}$. Each cell is identified by the string of
    moves on its path, with $\stay$ for a step without a split. The final row shows the cells $C_4(v)$ used for a new observation $X_4$. The highlighted cell belongs to probe $v=4$ and includes $N_3(v)=2$ splits, corresponding to $\NR_3(v)=\NL_3(v)=1$.}
    \label{fig:tree}
\end{figure}
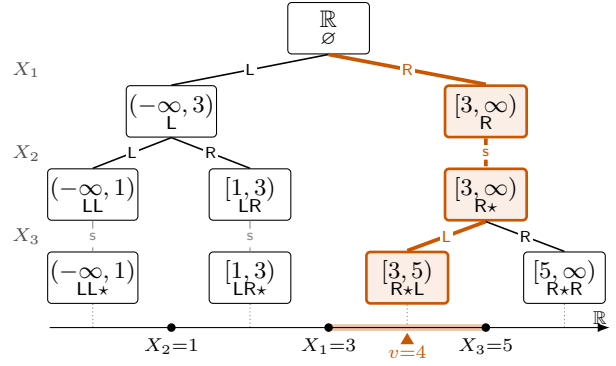

\paragraph{Proof idea.}
The cells in \eqref{eq:cellsR} form a data-driven partition that can be
organized as a tree, see Figure~\ref{fig:tree}. At step $s$, the cell $[l,r)$
whose interior contains $X_s$ divides into $[l,X_s)$ and $[X_s,r)$, with
symbols $\sL$ and $\sR$, while every other cell is kept, with symbol $\stay$.
Each cell is thus identified by the string of symbols on its path from the
root, and the partition used at step $s$ depends only on $X_{1:s-1}$, so it is
fixed before $X_s$ is observed.

The tree provides a way to organize the contributions to the deviations of the empirical CDF simultaneously over all $v$. At each step $s$, the increment of the deviation at $v$ is
$\mathbf 1\{X_s\le v\}-F_s(v)$. Let $[l,r)$ be the cell of $v$ before $X_s$
arrives. Once $X_s$ is observed, the indicator is constant on each child of
this cell, and $F_s(v)\ge F_s(l)$ because $F_s$ is nondecreasing. Hence the
envelope $\mathbf 1\{X_s\le v\}-F_s(l)$, which is the value of the indicator on
the child minus $F_s$ at the left endpoint of the parent, bounds the increment
uniformly over the child. The envelope sum along a path therefore provides an upper bound on the cumulative deviation for every point following that path. The charge scheme (Definition~\ref{def:charge}) assigns weights to the edges of the tree so that the weighted exponential contributions of all children of any given cell satisfy the conditional budget constraint \eqref{eq:star}. Multiplying these edge weights along each path and summing the resulting weighted exponential envelope sums over all cells at level $t$ yields $M_t$ defined in \eqref{eq:Mt}. The charge condition ensures that this aggregate is a nonnegative supermartingale, as established in Theorem~\ref{thm:master}. Section~\ref{sec:refinementschemes} formalises this argument in a general framework, with the stated result derived as a special case in Appendix~\ref{app:cdf}.
\begin{remark}\label{rem:invariance}
Let $\phi:\R\to\R$ be strictly increasing. The band of
Theorem~\ref{thm:cdfbandtrivial} computed from $(\phi(X_s))_{s\le t}$ at probe
$\phi(v)$ is numerically identical to the band computed from $(X_s)_{s\le t}$ at
$v$.
\end{remark}

\begin{remark}[Interpretation of $N_t(v)$]
The quantity $N_t(v)$ depends on the order in which the observations arrive and admits a natural interpretation in terms of binary search. Consider the binary search tree $T_t$ obtained by inserting the distinct values among $X_{1:t}$ in their order of first appearance. Then $N_t(v)$ is precisely the search depth of $v$ in this tree (move right if $x\leq v$), or equivalently, the number of comparisons required to locate $v$ given the observations seen up to time $t$. This also gives a useful interpretation of the price paid by the confidence band. At any time, only observations that fall strictly inside the current gap containing $v$ can further refine the location of $v$ and hence increase $N_t(v)$. An observation outside this gap leaves the cell containing $v$ unchanged and incurs no additional cost. Thus, $N_t(v)$ measures the amount of local refinement that has occurred around $v$, rather than the total number of observations.
\end{remark}

Since $N_t(v)$ can be identified with a search depth, bounds on it follow from the theory of random binary search trees whenever the observations arrive in uniformly random relative order.
\begin{proposition}[\cite{Luc,PITTEL1984461}]
\label{prop:typical}
Suppose the ranks of $X_{1:t}$ are a uniformly random permutation. For example, the $X_s$ are i.i.d.\ from a continuous law, or exchangeable with a.s.\ continuous de Finetti mixands. Then 
\[
\sup_vN_t(v)/\ln t\to c^{*}=4.311\ldots \quad \text{a.s.}
\]

\end{proposition}

\begin{remark}[A fixed $\lambda$ gives a non-vanishing width]
Under the assumptions of Proposition~\ref{prop:typical} the number of splits grows only logarithmically, so the second term of $\rho_t$ vanishes as $t\to\infty$. The first term $\lambda/8$ does not, so for fixed $\lambda>0$ the half-width stays bounded away from zero. Since the best $\lambda$ depends on $t$ and on the number of splits, which is random and differs
across probes, we take a union bound over a grid of values of $\lambda$ rather than fix one in advance.
\end{remark}
To obtain a vanishing half-width, $\lambda$ must decrease with the local complexity of the probe. We therefore consider a countable grid of values of $\lambda$, allocate the miscoverage probability across the grid, and apply a union bound over the resulting confidence bands. This allows the bound to adapt to the scale of the local complexity without fixing $\lambda$ in advance.

\begin{corollary}[$\lambda$-grid]\label{cor:mixture}
Let $\alpha\in(0,1)$. Let $\lambda_j=2^{-j}$ and $\alpha_j=\frac1{(j+1)(j+2)}\alpha$ for $j\ge0$. 
Applying Theorem~\ref{thm:cdfbandtrivial} with $(\lambda_j,\alpha_j)$ for
each $j\ge0$ and taking a union bound, with probability at least
$1-\alpha$ all of the resulting confidence bands hold simultaneously.
Consequently, taking the maximum of the lower bounds and the minimum of the
upper bounds over $j$ yields a valid confidence band. Its width is
bounded by
\begin{equation}
\label{eq:rate}
O\!\left(\sqrt{\frac{N_t(v)+\ln\ln t+\ln\frac1\alpha}{t}}\right).
\end{equation}
\end{corollary}

Under the assumptions of Proposition~\ref{prop:typical}, $N_t(v)=O(\log t)$ almost surely, uniformly in $v$, so by Corollary~\ref{cor:mixture} the band shrinks at the rate
\[
O\left(\sqrt{\frac{\ln t}{t}}\right)
\]
almost surely. The dependence on $N_t(v)$ also makes the role of the sampling scheme explicit. In the worst case an adversary who always places $X_s$ inside the current
cell containing a target $v$ drives $N_t(v)=t$, and then the half-width exceeds $1$ for every $\lambda$, so the bound is vacuous at $v$. The shrinking of the band is thus tied not to the number of observations but to the rate at which the local count $N_t(v)$ grows.

\begin{remark}[Choice of $\lambda$ throughout]\label{rem:grid}
All subsequent results are stated for a fixed $\lambda$. Each of them is combined over the grid of Corollary~\ref{cor:mixture} in the same way, by applying it with $(\lambda_j,\alpha_j)$ for every $j$ and intersecting the resulting bands, and we do not repeat this step. For results that require $\lambda<1$, the grid starts at $j=1$, with $\alpha_j=\frac{\alpha}{j(j+1)}$, which again sums to $\alpha$.
\end{remark}

\subsection{Variance adaptivity}

The construction in Theorem~\ref{thm:cdfbandtrivial} relies on a Hoeffding bound for the moment-generating function (MGF) of a Bernoulli random variable. This bound is convenient but conservative, as it does not retain the dependence on the Bernoulli variance and hence does not yield variance-adaptive confidence bounds. In this subsection, we instead evaluate the Bernoulli MGF exactly, which allows us to exploit this variance dependence. 
The difficulty is that the Bernoulli parameter is $\mathcal{F}_{s-1}$-measurable but unknown. Hoeffding's bound circumvents this issue by providing a uniform bound over all possible values of the parameter, but in doing so loses this dependence.

Our exact-MGF argument instead retains the parameter throughout and applies Jensen's
inequality to the concave map $p \mapsto \ln\bigl(1 + p(e^{\lambda} - 1)\bigr)$. The bound then depends on the unknown parameters only through
their average, which recovers variance adaptivity without requiring knowledge of the
parameter.

\begin{restatable}{theorem}{cdfbandvar}
\label{thm:cdfbandvar}
Let the miscoverage level be $\alpha\in(0,1)$ and set $\lambda > 0$. Then the
confidence band given by
\begin{align*}
    \Flo_t(v) &:=\frac{\exp\bigl\{\lambda\hat q_t(v)-c_t\bigl(N_t(v)\bigr)\bigr\}-1}{e^{\lambda}-1},\\
    \Fhi_t(v) &:= 1-\frac{\exp\bigl\{\lambda\bigl(1-\hat q'_t(v)\bigr)-c_t\bigl(N_t(v)\bigr)\bigr\}-1}{e^{\lambda}-1},
\end{align*}
where
\[
\hat q_t(v):=\Fhat_t(v)-\frac{\NR_t(v)}{t},
\qquad
\hat q'_t(v):=\hF_t(v)+\frac{\NL_t(v)}{t},
\]
and
\[
c_t(n):=\frac{n\ln 2+\ln\frac{2}{\alpha}}{t},
\]
is valid uniformly over time and $v$, i.e., it satisfies
\eqref{eq:general-goal} with $\mu_t=\Fbar_t$.
\end{restatable}

Remark~\ref{rem:varadapt} explains why the band adapts to the variance of the time-averaged distribution.

\section{REFINEMENT AND CHARGE SCHEMES}\label{sec:refinementschemes}

We now formulate a framework in which we can derive general confidence bands for more general functional \eqref{eq:general-target} from a stream of observations. Let $(\Omega,\Ft,(\Ft_s)_{s\ge0},\PP)$ be a filtered probability space, write
$\Es[\cdot]:=\EE[\cdot\mid\Ft_{s-1}]$, and let $(\V,\Sigma_\V)$ be a
nonempty measurable space of probes. For notational ease, we write the observable function $\probefunction(X_s,v)$ at time $s$ as $\probefunction_s(v)$. We can derive the confidence bands that satisfy \eqref{eq:general-goal} by bounding 
\[D_t(v)=\sum_{s\le t}(\probefunction_s(v)-\theta_s(v)),\] with
$\theta_s(v)=\Es [\probefunction_s(v)]$, simultaneously over all $t$ and all $v\in\V$.

\begin{remark}
We allow $\probefunction_s$ to depend on the current observation $X_s$, as well as on the past observations $X_{1:s-1}$. In particular, $\probefunction_s$ need not be the same function at each time $s$ and may adapt to the history of the process.
\end{remark}

To do so, we recursively partition the probe space $\V$ as observations arrive. The partition is encoded using a countable alphabet $\Alp$. As an example, consider the partition of $\V=\mathbb R$ illustrated in Figure~\ref{fig:tree}. Each string identifies a cell and when a cell is split by a new observation, the symbols in $\Alp=\{\sL,\sR,\stay\}$ denote, respectively, the portions to the left of and to the right of the observation, while $\stay$ indicates that the cell stays unchanged. Thus the string $a=\sR\stay\sL$ corresponds to the interval $[3,5)$. Strings are vertices of a rooted $\Alp$-ary tree. In addition, $\varnothing$ is the empty string, $ak$ denotes concatenation with an additional symbol $k \in \Alp$, and $a_{\le n}$ is the length-$n$
prefix.

The preceding example motivates the following definition, in
which strings encode the successive refinements of the partition.
\begin{definition}[Refinement scheme]
\label{def:scheme}
Let $\Alp$ be a countable alphabet containing a distinguished symbol $\stay$.
A \emph{refinement scheme} on $(\V,(\Ft_s))$ is a family of maps
$a_s:\Omega\times\V\to\Alp^{s-1}$, $s\ge1$, such that
\begin{enumerate}
    \renewcommand{\labelenumi}{(R\theenumi)}
    \item $a_1(\omega,v)=\varnothing$ for all $\omega,v$;
    \item $a_s(\omega,v)$ is a prefix of $a_{s+1}(\omega,v)$;
    \item for every $a$ with $|a|=s-1$, 
    $\{(\omega,v):a_s(\omega,v)=a\}\in
    \Ft_{s-1}\otimes\Sigma_\V$.
    \item for every $s$ and $a\in \Alp^{s-1}$, on $\{C_s(a)\neq\emptyset\}$
either $C_{s+1}(a\stay)=C_s(a)$ or $C_{s+1}(a\stay)=\emptyset$.
\end{enumerate}
Here $C_s(a):=\{v:a_s(\omega,v)=a\}$ is the \emph{cell} of $a$, $C_s(v):=C_s(a_s(v))$,
and $\cC_s$ is the set of nonempty cells $C_s(a)$.
In what follows, we suppress the first argument $\omega\in\Omega$ and write $a_s(\cdot)$ for the realization $a_s(\omega,\cdot)$ at a fixed $\omega\in\Omega$. 
Furthermore, we write $k_s(v)\in\Alp$ for the symbol emitted at step $s$, so
$a_{s+1}(v)=a_s(v)k_s(v)$, and we let
$\cS_t(v):=\{s\le t:\ k_s(v)\neq\stay\}$ be the set of steps at which the cell
of $v$ divides.
\end{definition}

The first three conditions above say: \emph{start together}; \emph{only split, never merge};
\emph{know the current partition before seeing the next observation}. A cell
that does not divide emits a distinguished symbol and reappears one level down,
so exactly one symbol is emitted per step and each $v$ traces a path along the tree.
By the third condition, $\Cells_s$ is an at most countable, $\Ft_{s-1}$-measurable partition of
$\V$, and $\Cells_{s+1}$ refines $\Cells_s$.
\begin{assumption}
\label{ass:meas}
For all $s\ge1$ and $a\in\Alp^{s-1}$, the event $\{C_s(a)\neq\emptyset\}$ is
$\Ft_{s-1}$-measurable.
\end{assumption}

\begin{definition}[Envelope]\label{def:envelope}
Fix an $\mathcal{F}_{s-1}$-measurable version $\theta_s(v)$ of $\E_{s-1}[f_s(v)]$ for every $s\ge1$, $v\in \mathcal{V}$. For $s\ge1$ let $\probefunction_s:\Omega\times\V\to\R$ be
$\Ft_s\otimes\Sigma_\V$-measurable with $\theta_s(v)$ finite. An
\emph{envelope} is a family of $\Ft_s$-measurable
$\bar\probefunction_s(b)\in[-\infty,\infty)$, $b\in\Alp^s$, such that
$\bar\probefunction_s(b)=-\infty$ if $C_{s+1}(b)=\emptyset$ and, almost surely,
\[
  \bar\probefunction_s(b)\ \ge\ \probefunction_s(v)-\theta_s(v)
  \qquad\text{for all }v\in C_{s+1}(b).
\]
The \emph{envelope sum} is
$\tilde D_t(a):=\sum_{s\le t}\bar\probefunction_s(a_{\le s})$, with
$e^{-\infty}:=0$.
\end{definition}

The smallest envelope is
$\sup_{v\in C_{s+1}(b)}(\probefunction_s(v)-\theta_s(v))$. The charge schemes
of Appendix~\ref{sec:templates} use explicit envelopes that are easier to
control.

\begin{lemma}
\label{lem:dom}\label{lem:DleqDt}
Almost surely, $D_t(v)\le\tilde D_t(a_{t+1}(v))$ for all $t\ge0$ and $v\in\V$.
\end{lemma}

At step $s$, a parent cell indexed by $a$ gives rise to children $ak$, and we now consider assigning a weight or `charge' to each of their corresponding edges.

\begin{definition}[Charge scheme]
\label{def:charge}
Fix $\lambda>0$ and an envelope. A \emph{charge scheme at $\lambda$} is a family of $\Ft_s$-measurable variables $c_s(ak) : \Alp^{s} \rightarrow \mathbb{R}_+$ such that
\begin{equation}
\label{eq:star}
\Es\Big[\textstyle\sum_{k\in\Alp}c_s(ak)e^{\lambda\bar \probefunction_s(ak)}\Big]\le1
\end{equation}
a.s. on $\{C_s(a)\ne\emptyset\}$ for every
$s$ and every $a\in\Alp^{s-1}$.
We also define $w_t(a):=\prod_{s\le t}c_s(a_{\le s})$, with $w_0(\varnothing):=1$. 
\end{definition}

The quantity $w_t(a)$ is the product of the charges along the path from the root to $a$. Condition \eqref{eq:star} is imposed \emph{per parent} and summed over its children. The cost of covering an infinite probe set is thus paid one node at a time. Charges may be $\Ft_s$-measurable rather than merely
predictable.
\begin{theorem}[Master supermartingale]
\label{thm:master}
Let $(c_s)$ be a charge scheme at $\lambda>0$. Then

\begin{equation}\label{eq:Mt}
M_t:=\sum_{a\in\Alp^t}w_t(a)e^{\lambda\tilde D_t(a)} 
\end{equation}
is a nonnegative supermartingale with $M_0=1$.
\end{theorem}

The proof is a short conditional expectation computation. The term of each parent in $M_{s-1}$ is passed on to its children, and \eqref{eq:star} ensures that, in conditional expectation, the children together receive no more than their parent. See App.~\ref{app:master}.

\begin{corollary}[Master band]
\label{cor:masterband}
For $\alpha\in(0,1)$ and any charge scheme at $\lambda>0$, with probability at
least $1-\alpha$, simultaneously for all $t\ge1$ and all $v\in\V$,
\[
\sum_{s\le t}\big(\probefunction_s(v)-\theta_s(v)\big)\;\le\; \tilde D_t\bigl(a_{t+1}(v)\bigr)\;\le\;
\frac{\Lcode_t(v)+\ln\frac1\alpha}{\lambda},
\]
where 
$\Lcode_t(v):=\ln\frac1{w_t(a_{t+1}(v))}\in(-\infty,\infty]$.
\end{corollary}

Theorem~\ref{thm:cdfbandtrivial} follows directly from
Corollary~\ref{cor:masterband} by using the refinement scheme of
Definition~\ref{def:cellsR} (illustrated in Figure~\ref{fig:tree}). For a cell
$[l,r)$, take the envelope $\bar\probefunction_s(ak)=\probefunction_s[ak]-F_s(l)$,
where $\probefunction_s[ak]$ is the value of $\ind\{X_s\le v\}$ on the child
$ak$, and the charge scheme $c_s(a\stay)=e^{-\lambda^2/8}$,
$c_s(a\sL)=\frac12e^{-\lambda^2/8}$ and $c_s(a\sR)=\frac12e^{-\lambda^2/8-\lambda}$.
More generally, Appendix~\ref{sec:templates} gives two generic charge schemes that produce all bands of this paper. 

Corollary~\ref{cor:masterband} bounds $\sum_{s\le t}(f_s(v)-\theta_s(v))$ from above,
and hence gives a lower confidence bound on $\mu_t(v)$. The upper bound needs a
second application of Corollary~\ref{cor:masterband}, to a family whose deviations
are those of $f_s$ with the sign reversed. Thus confidence bands $(\Flo_t,\Fhi_t)$ are readily obtained from Corollary~\ref{cor:masterband}. The result turns the problem of inferring a function \eqref{eq:general-target} of interest into one of constructing a suitable pair of refinement and charge schemes.

\section{OTHER DISTRIBUTION FUNCTIONS}
We now use the framework of Section~\ref{sec:refinementschemes} to obtain confidence bands for other distribution functions.
\subsection{Counterfactual CDF}\label{sec:counterfactual-cdf}

In many sequential decision problems, the distribution of interest is counterfactual, corresponding to a policy that differs from the data-collection process. To address this distribution shift, we use importance weighting to construct a confidence band for the counterfactual CDF.\\
 At
time $s$ we observe $(W_s,X_s)$ with $W_s\ge0$ and $\Es[W_s]=1$, e.g.\ $W_s$
the likelihood ratio between the target and logging policies; no upper bound
on $W_s$ is assumed. Then $$F^W_s(v):=\Es[W_s\ind\{X_s\le v\}]$$ is the
conditional CDF of $X_s$ under the target policy, and we seek bands for
$\bF^W_t(v):=\frac1t\sum_{s\le t}F^W_s(v)$ based on the importance-weighted
empirical CDF $\hF^W_t(v):=\frac1t\sum_{s\le t}W_s\ind\{X_s\le v\}$.

In the framework of Section~\ref{sec:refinementschemes}, this is the probe
function $\probefunction_s(v)=W_s\ind\{X_s\le v \}$ with $\theta_s=F^W_s$, so
$D_t(v)=t\bigl(\hF^W_t(v)-\bF^W_t(v)\bigr)$. We reuse the refinement scheme of
Section~\ref{sec:cdf} (Definition~\ref{def:cellsR}, Figure~\ref{fig:tree}) and
change only the charges.
\begin{definition}\label{def:pbet}
Fix $\lambda\in(0,1)$ and put $\eta(x):=x-\ln(1+x)$ for $x>-1$. For
$a\in\Alp^{s-1}$ with $C_s(a)\ne\emptyset$, let $l:=\inf C_s(a)$ and
$r:=\sup C_s(a)$, and put
\begin{align*}
  A&:=F^W_s(l), & Z_s&:=W_s\ind\{X_s\le l\},\\
  A^>&:=\Es[W_s\ind\{X_s\ge r\}], & Z^>_s&:=W_s\ind\{X_s\ge r\},
\end{align*}
with $A:=0$ if $l=-\infty$ and $A^>:=0$ if $r=+\infty$. The lower and upper charges are defined, respectively, as follows:
\begin{align*}
  c_s(a\stay)&:=e^{-\eta(\lambda(Z_s-A))},
  &c_s^>(a\stay)&:=e^{-\eta(\lambda(Z_s^>-A^>))},\\
  c_s(a\sL)&:=\tfrac12e^{-\eta(-\lambda A)},
  &c_s^>(a\sL)&:=\tfrac12e^{-\eta(-\lambda A^>)-\lambda W_s},\\
  c_s(a\sR)&:=\tfrac12e^{-\eta(-\lambda A)-\lambda W_s},
  &c_s^>(a\sR)&:=\tfrac12e^{-\eta(-\lambda A^>)}.
\end{align*}
\end{definition}

Here $A$ and $A^>$ are the smallest conditional means of
$W_s\ind\{X_s\le v\}$ and of $W_s\ind\{X_s>v\}$ over the cell, and $Z_s$ and
$Z^>_s$ are the observed values on the cell when it is not split. 
\begin{theorem}\label{thm:wfv}
Let $\alpha,\lambda\in(0,1)$ and
$\kappa(\lambda):=\bigl(-\lambda-\ln(1-\lambda)\bigr)/\lambda$. For $s\le t$,
let $\check f_s(v):=W_s\ind\{X_s\le v\}$ and $\check g_s(v):=W_s\ind\{X_s>v\}$
if $v$ is not split at time $s$, and $\check f_s(v):=\check g_s(v):=0$ if it
is. Put
\begin{align*}
  \check F^W_t(v)&:=\frac1t\sum_{s\le t}\check f_s(v), & \widehat V_t(v)&:=\sum_{s\le t}\check f_s(v)^2,\\
  \check F^{W,>}_t(v)&:=\frac1t\sum_{s\le t}\check g_s(v), & \widehat V^>_t(v)&:=\sum_{s\le t}\check g_s(v)^2,
\end{align*}
and $\rho_t(v):=\bigl(N_t(v)\ln2+\ln\frac2\alpha\bigr)/(\lambda t)$. Then the
band given by
\begin{align}
  \Flo^W_t(v)&:=\frac{\check F^W_t(v)-\kappa(\lambda)\widehat V_t(v)/t-\rho_t(v)}{1+\kappa(\lambda)},\label{eq:wfv-lower}\\
  \Fhi^W_t(v)&:=1-\frac{\check F^{W,>}_t(v)-\kappa(\lambda)\widehat V^>_t(v)/t-\rho_t(v)}{1+\kappa(\lambda)},\label{eq:wfv-upper}
\end{align}
is valid uniformly over time and $v$, that is, with probability at least
$1-\alpha$, $\Flo^W_t(v)\le\bF^W_t(v)\le\Fhi^W_t(v)$ for all $t\ge1$ and
$v\in\R$.
\end{theorem}

The charges depend on the unknown centerings $A$ and $A^{>}$, but the band does
not. Fan's inequality removes them, at the price of the term
$\kappa(\lambda)\widehat V_t(v)/t$ and the factor $1/(1+\kappa(\lambda))$
(Corollary~\ref{cor:iw-explicit}). Replacing $A$ and $A^{>}$ by predictable
guesses, such as running means, gives bands of a different form
(Theorem~\ref{thm:iw-two}; see Remark~\ref{rem:centering} for more
examples).

\subsection{Right-censored data}\label{sec:censored}
In survival analysis, the lifetime or time-to-event $T_s$ of subject $s$ is often right-censored
by a censoring time $\cens_s$. The (conditional) survival function is
\[S_s(v):=\Es[\ind\{T_s > v\}],\]
and fix a horizon $\tau$. Suppose that, given the past, $\mathsf C_s$
is independent of $T_s$, its law is known, and
$\mathbb P_{s-1}(\mathsf C_s\ge\tau)>0$ (Assumption~\ref{ass:cens}). We only observe $\tT_s=T_s\wedge \cens_s$ and whether
the event was seen, $\delta_s=\ind\{T_s\le \cens_s\}$. Our framework gives two-sided anytime-valid
bands on the running average $\bS_t :=\frac1t\sum_{s\le t}S_s(v)$ of the conditional survival functions,
uniformly over time and $v\in(0,\tau)$ and without assuming i.i.d.\ lifetimes. The
upper bound follows from Section~\ref{sec:counterfactual-cdf} by weighting each
observed event by the inverse of its probability of being observed. For the
lower bound, we instead weight each subject still under observation at the
threshold by the inverse of its probability of remaining under observation that
long. Appendix~\ref{app:censored} gives the construction, the band
(Theorem~\ref{thm:censband}) and its proof.

\subsection{Multivariate CDF}\label{sec:mcdf}
Let $X_s\in\R^d$ and write $u\le v$ componentwise. The target is the running
average $\Fb_t(v):=\frac1t\sum_{s\le t}F_s(v)$ of the conditional joint CDFs
$F_s(v):=\PP(X_s\le v\mid\Ft_{s-1})$, and $\hF_t(v):=\frac1t\sum_{s\le t}\ind\{X_s\le v\}$
is the empirical joint CDF. We use the box refinement of
Appendix~\ref{app:mcdf}. Its cells are half-open boxes, and an observation
$X_s$ cuts the box containing a probe only in the coordinates in which $X_s$
lies strictly inside the box (Figure~\ref{fig:boxes}). This makes
$v\mapsto\ind\{X_s\le v\}$ constant on every new box. For $d=1$ it is the
refinement of Definition~\ref{def:cellsR}.

Let $N_t(v)$ be the number of steps at which the box of $v$ is cut, and
$M_t(v)$ the total number of coordinates in which it is cut. Let $N^{\le}_t(v)$
count the cuts with $X_s\le v$, and $N^{\not\le}_t(v):=N_t(v)-N^{\le}_t(v)$.

\begin{theorem}\label{thm:mcdf}
Let $\alpha\in(0,1)$ and $\lambda>0$, and put
\[
  \rho_t(v):=\frac\lambda8+\frac{M_t(v)\ln2+\ln\frac2\alpha}{\lambda t}.
\]
With probability at least $1-\alpha$, simultaneously for all $t\ge1$ and
$v\in\R^d$,
\begin{align*}
  \Fb_t(v)&\ \ge\ \hF_t(v)-\frac{N^{\le}_t(v)}t-\rho_t(v),\\
  \Fb_t(v)&\ \le\ \hF_t(v)+\frac{N^{\not\le}_t(v)}t+\rho_t(v).
\end{align*}
\end{theorem}

Each coordinate cut costs $\ln2$, exactly as one split in dimension one, and
for $d=1$ the theorem above resolves into Theorem~\ref{thm:cdfbandtrivial}. The band is
invariant under coordinatewise strictly increasing transformations of $X_s$,
and it needs no assumption on the dependence between observations.

\section{NUMERICAL EXPERIMENTS}\label{sec:experiments}
We compare our bands, tightened by the closure of
Proposition~\ref{prop:closure}, with those of \citet{mineiro2023} (MH),
computed with the authors' reference implementation (Appendix~\ref{app:experiments}). Furthermore, we illustrate the censored case.  Throughout, $\alpha=0.05$ and each panel shows a single run (two in Figure~\ref{fig:cdf}(c)). Appendix~\ref{app:experiments} gives the details. We apply the closure only to our bands. It leaves the MH bands essentially
unchanged (Remark~\ref{rem:closure-mh}).

\begin{figure*}[t]
    \centering
\includegraphics[width=\textwidth]{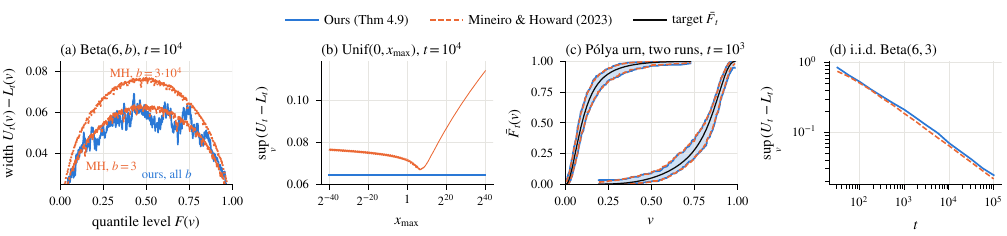}
    \caption{Bands for $\overline F_t$. (a) Width at $t=10^4$ against the quantile
level, for i.i.d.\ Beta$(6,b)$ data obtained from the same uniforms. (b)
Sup-width at $t=10^4$ for i.i.d.\ Unif$(0,x_{\max})$ data. (c) Bands at $t=10^3$ for two runs of the P\'olya urn (ours
shaded). (d) Sup-width against $t$ for i.i.d.\ Beta$(6,3)$ data.}
    \label{fig:cdf}
\end{figure*}

\paragraph{CDF.}
Figure~\ref{fig:cdf} bounds the running average $\overline F_t(v)$ of the conditional CDFs with Theorem~\ref{thm:cdfbandvar}. The data are i.i.d.\ Beta$(6,b)$ in (a) and (d), and i.i.d.\ Unif$(0,x_{\max})$ in (b). In (c) they come from the continuous P\'olya urn of MH, for which $\overline F_t$ changes with $t$. MH refine a dyadic grid fixed in advance, on $[0,1]$ through a reference CDF,
or on $\mathbb R$ in (b), paying for each level of the grid in a union bound.
Data concentrated far from the scale of the grid therefore need more levels. Panels (a) and (b) vary this concentration. Within each panel, all samples are
strictly increasing transformations of the same uniforms. By
Remark~\ref{rem:invariance}, our band at a given quantile level is therefore
identical for every $b$ and $x_{\max}$, and only the MH band changes. With
independent samples, our band would change only through sampling noise, because
for i.i.d.\ continuous data the law of the relative order of the observations
does not change.
\begin{figure}[t]
    \centering
    \includegraphics[width=\columnwidth]{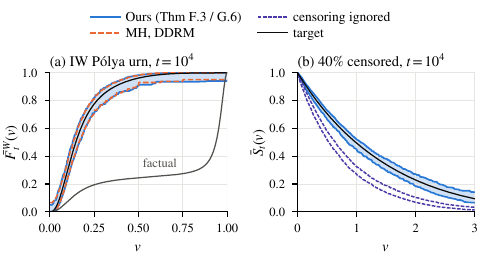}
    \caption{(a) Bands at $t=10^4$ for the counterfactual $\overline F^W_t$ in the
    importance-weighted P\'olya urn (ours shaded), with the factual running-average
    CDF. (b) Bands at $t=10^4$ for $\overline S_t$ with 40\% of the subjects censored
    (ours shaded), and the band that ignores the censoring.}
    \label{fig:ope}
\end{figure}
\paragraph{Importance weights and censoring.}
Figure~\ref{fig:ope}(a) bounds the running average $\overline F^W_t$ of the counterfactual conditional CDFs using Theorem~\ref{thm:iw-two}, based on one of the examples from MH, their importance-weighted P\'olya urn.
 Figure~\ref{fig:ope}(b) bounds the running average $\overline S_t$ of the conditional survival functions with Theorem~\ref{thm:cens-two}. Here subjects arrive in a stream, their risk depends on the outcomes of earlier subjects, and their lifetimes are censored at random times with a known law.
\section{DISCUSSION}
We constructed confidence bands for the running average of conditional distribution functions that hold simultaneously over all times and all points, without assumptions on the dependence between observations. Uniformity over an infinite index set is obtained from a single supermartingale on a data-driven refinement, in which each probe pays only for the refinements of its own cell. For thresholds this cost is the search depth of the probe in the binary search tree of the observations. Because the cells depend only on the relative order of the data, the bands are invariant under any fixed, strictly increasing transformation, which answers a question raised by \citet{mineiro2023}. The same framework covers importance-weighted, right-censored and multivariate distribution functions.

The width guarantee is instance dependent. It depends on the order in which observations arrive, so trending streams can inflate the cost of uniformity, and under exchangeable data the rate $\sqrt{\log t/t}$ is slower than the iterated-logarithm rate available in the i.i.d.\ setting. Whether this gap is necessary for bands to remain valid under arbitrary dependence is open.

We believe our work can be extended in several interesting ways. In one case-study, the results are derived for right-censored data, but may also extend to other forms of censoring. Moreover, while our applications focus on relatively simple settings, the target in Eq. \eqref{eq:general-target} is quite general and may yield bounds for loss functions arising in both parametric and nonparametric estimation. Identifying the most effective refinements in these settings requires further study. Finally, our sequential refinement scheme suggests natural connections to coding and information theory that merit further exploration.

\newpage

\subsection*{AI use statement}
In this work, we used generative AI tools as a discussion partner during the development of the theory, and to implement the methods used in the experiments. We have not used generative AI tools to interpret results, and translation, cleaning and reformatting data sets, and qualitative and thematic data analysis are not applicable to this work. Additionally, we used generative AI tools to generate code and figures and to edit the paper to
improve readability. We have reviewed all AI-assisted work. All proofs and code were checked by one of the authors. We take responsibility for the final content of this work, including text, claims, or artifacts produced with the aid of generative AI.

\bibliographystyle{apalike}
\bibliography{refs}

\clearpage
\appendix
\thispagestyle{empty}

% Supplementary material: To improve readability, you must use a single-column format for the supplementary material.
\onecolumn
\aistatstitle{
APPENDIX}

\makeatletter
\@ifpackageloaded{hyperref}{}{\let\addcontentsline\LaTeXaddcontentsline}
\makeatother
\addtocontents{toc}{\protect\setcounter{tocdepth}{2}}
\renewcommand{\contentsname}{Contents of the Appendix}
\tableofcontents
\clearpage  % optional
\section{PROOF ROADMAP}
Figure~\ref{fig:deps} shows, for each result of the main text, the results
that its proof uses.
\begin{figure}[H]
\centering
\begin{tikzpicture}[
  >=Stealth,
  dep/.style={draw, rounded corners=2pt, align=center, font=\scriptsize,
              inner sep=2pt, minimum height=7.5mm, text width=23mm},
  mainres/.style={dep, very thick, fill=black!10},
  hd/.style={anchor=west, font=\small\bfseries, text=black},
  el/.style={font=\tiny, fill=white, inner sep=1pt},
  every edge/.style={draw, ->, thick}
]
% ---------------- (a) framework and templates ----------------
\node[hd] at (-0.1,2.7) {(a) Framework (Section~\ref{sec:refinementschemes}) and charge templates (Appendix~\ref{sec:templates})};
\node[dep]     (Ameas) at (1.1,0)    {Assumption~\ref{ass:meas}\\nonempty cells};
\node[mainres] (Tm)    at (3.9,0)    {Theorem~\ref{thm:master}\\master supermartingale};
\node[mainres] (Ldom)  at (3.9,-1.4) {Lemma~\ref{lem:dom}\\domination};
\node[mainres] (Cmb)   at (6.7,-0.7) {Corollary~\ref{cor:masterband}\\master band};
\node[dep]     (Lbern) at (9.5,1.8)  {Lemma~\ref{lem:bern}\\Bernoulli cumulant};
\node[dep]     (PA)    at (9.5,0.4)  {Proposition~\ref{prop:A}\\Template A charges};
\node[dep]     (Cbs)   at (12.3,0.4) {Corollary~\ref{cor:boundedsimple}\\path cost};
\node[dep]     (CA)    at (15.1,0.4) {Corollary~\ref{cor:A}\\Template A bands};
\node[dep]     (Afl)   at (12.3,-0.8){Assumption~\ref{ass:floor}\\cell floor};
\node[dep]     (PB)    at (9.5,-2.0) {Proposition~\ref{prop:B}\\Template B charges};
\node[dep]     (CB)    at (12.3,-2.0){Corollary~\ref{cor:B}\\path cost};
\node[dep]     (Cbpsi) at (15.1,-2.0){Corollary~\ref{cor:Bpsi}\\homogeneous form};

\path (Ameas) edge (Tm) (Tm) edge (Cmb) (Ldom) edge (Cmb)
      (PA) edge (Cbs) (Cmb) edge (Cbs) (Cbs) edge (CA)
      (Lbern) edge (PA) (Lbern) edge node[el,pos=0.75]{Jensen} (CA)
      (Afl) edge (PA) (Afl) edge (Cbpsi)
      (PB) edge (CB) (Cmb) edge (CB) (CB) edge (Cbpsi);

% ---------------- (b) Theorem 4.2 and Corollary 4.7 ----------------
\node[hd] at (-0.1,-3.5) {(b) Theorem~\ref{thm:cdfbandtrivial} and Corollary~\ref{cor:mixture}};
\node[dep]     (bCA)  at (1.1,-4.3)  {Corollary~\ref{cor:A}(i)\\Hoeffding form};
\node[dep]     (bThr) at (1.1,-5.4)  {Lemma~\ref{lem:threshold}\\half-open cells};
\node[dep]     (bFam) at (3.9,-5.4)  {Lemma~\ref{lem:families}\\homogeneity, floor};
\node[mainres] (bT)   at (6.7,-4.85) {Theorem~\ref{thm:cdfbandtrivial}\\CDF band};
\node[mainres] (bMix) at (9.5,-4.85) {Corollary~\ref{cor:mixture}\\$\lambda$-grid, union bound};

\path (bCA) edge (bT) (bThr) edge (bFam) (bFam) edge (bT)
      (bT) edge (bMix);

% ---------------- (c) Theorem 4.9 ----------------
\node[hd] at (11.2,-3.5) {(c) Theorem~\ref{thm:cdfbandvar}};
\node[dep]     (cCA)  at (12.3,-4.3) {Corollary~\ref{cor:A}(ii)\\exact form};
\node[dep]     (cFam) at (12.3,-5.4) {Lemmas~\ref{lem:threshold}, \ref{lem:families}\\half-open cells};
\node[mainres] (cT)   at (15.1,-4.85){Theorem~\ref{thm:cdfbandvar}\\variance-adaptive};

\path (cCA) edge (cT) (cFam) edge (cT);

% ---------------- (d) Theorem 6.2 ----------------
\node[hd] at (-0.1,-6.6) {(d) Theorem~\ref{thm:wfv}};
\node[dep]     (dB)   at (1.1,-7.4)  {Corollary~\ref{cor:Bpsi}\\Template B};
\node[dep]     (dFam) at (1.1,-8.5)  {Lemmas~\ref{lem:threshold}, \ref{lem:families}\\half-open cells};
\node[dep]     (dL)   at (3.9,-7.95) {Theorem~\ref{thm:iw-lower}\\lower, any centering};
\node[dep]     (d2)   at (6.7,-7.95) {Theorem~\ref{thm:iw-two}\\two-sided};
\node[dep]     (dX)   at (9.5,-7.95) {Corollary~\ref{cor:iw-explicit}\\oracle centering};
\node[font=\tiny, text=black] at (8.1,-7.45) {$\Es W_s=1$};
\node[dep]     (dFan) at (9.5,-9.25) {Lemma~\ref{lem:fan}\\Fan's inequality};
\node[mainres] (dT)   at (12.3,-7.95){Theorem~\ref{thm:wfv}\\weighted CDF, weights $\frac12,\frac12$};

\path (dB) edge (dL) (dFam) edge (dL) (dL) edge (d2)
      (d2) edge (dX) (dFan) edge (dX)
      (dX) edge (dT);

% ---------------- (e) censored ----------------
\node[hd] at (-0.1,-10.3) {(e) Theorem~\ref{thm:censband} (Section~\ref{sec:censored})};
\node[dep]     (eL)   at (1.1,-11.1) {Theorem~\ref{thm:iw-lower}\\(its argument)};
\node[dep]     (eM2)  at (1.1,-12.2) {Lemma~\ref{lem:cens-moments}(ii)\\event weights};
\node[dep]     (eU)   at (3.9,-11.65){Theorem~\ref{thm:cens-upper}\\upper bound};
\node[dep]     (eB)   at (3.9,-12.9) {Corollary~\ref{cor:B}\\Template B};
\node[dep]     (eThr) at (1.1,-13.5) {Lemma~\ref{lem:threshold}(iii)\\half-open cells};
\node[dep]     (eM3)  at (1.1,-14.6) {Lemma~\ref{lem:cens-moments}(i),(iii)\\at-risk moments};
\node[dep]     (eAt)  at (3.9,-14.05){Lemma~\ref{lem:atrisk}\\at-risk probe};
\node[dep]     (eLo)  at (6.7,-13.45){Theorem~\ref{thm:cens-lower}\\lower bound};
\node[mainres] (eT)   at (9.5,-12.55){Theorem~\ref{thm:censband}\\censored band, union bound};

\path (eL) edge (eU) (eM2) edge (eU)
      (eThr) edge (eAt) (eM3) edge (eAt)
      (eB) edge (eLo) (eAt) edge (eLo)
      (eU) edge (eT) (eLo) edge (eT);

% ---------------- (f) Theorem 6.3 ----------------
\node[hd] at (11.2,-10.3) {(f) Theorem~\ref{thm:mcdf}};
\node[dep]     (fCA)  at (12.3,-11.1){Corollary~\ref{cor:A}(i)\\Hoeffding form};
\node[dep]     (fBox) at (12.3,-12.2){Lemma~\ref{lem:boxes}\\box refinement};
\node[mainres] (fT)   at (15.1,-11.65){Theorem~\ref{thm:mcdf}\\multivariate CDF};

\path (fCA) edge (fT) (fBox) edge (fT);

\end{tikzpicture}
\caption{Dependencies of the proofs. An arrow from one result to another
means that the proof of the second uses the first. Shaded boxes are results
stated or announced in the main text. Panel (a) shows the shared machinery, namely the
master band and the two charge templates, with Template A for bounded
families in the upper row and Template B for nonnegative families in the
lower row. Panels (b) to (f) start from these templates.}
\label{fig:deps}
\end{figure}
\newpage
\section{MISSING PROOFS \S\ref{sec:refinementschemes}}
\label{app:master}
\begin{proof}[Proof of Lemma~\ref{lem:dom}]
Fix $v$ and $s\le t$. By (R2), $a_{s+1}(v)$ is the length-$s$ prefix of
$a_{t+1}(v)$, and $v\in C_{s+1}(a_{s+1}(v))$. Hence
$\probefunction_s(v)-\theta_s(v)\le\bar \probefunction_s(a_{s+1}(v))=\bar \probefunction_s((a_{t+1}(v))_{\le s})$.
Summing over $s\le t$ gives the claim.
\end{proof}

\begin{proof}[Proof of Theorem~\ref{thm:master}]
Nonnegativity is clear and $M_0=w_0(\varnothing)e^0=1$. The sum is countable and
each summand is $\Ft_t$-measurable: $w_t(a)$ is a product of
$\Ft_s$-measurable factors with $s\le t$, and $\tilde D_t(a)$ a sum of such by
Definition~\ref{def:envelope}. Fix $s\ge1$ and set, for $a\in\Alp^{s-1}$,
$U(a):=w_{s-1}(a)e^{\lambda\tilde D_{s-1}(a)}\ge0$, which is
$\Ft_{s-1}$-measurable. Splitting each level-$s$ address as $ak$,
\begin{align*}
M_s&=\sum_{a\in\Alp^{s-1}}\sum_{k\in\Alp}w_{s-1}(a)c_s(ak)\,
e^{\lambda\tilde D_{s-1}(a)+\lambda\bar \probefunction_s(ak)}\\
&=\sum_{a\in\Alp^{s-1}}U(a)\sum_{k\in\Alp}c_s(ak)e^{\lambda\bar \probefunction_s(ak)} .
\end{align*}
All terms are nonnegative, so conditional monotone convergence and pulling out
the $\Ft_{s-1}$-measurable factor $U(a)$ give
\[
\Es[M_s]=\sum_{a}U(a)\,
\Es\Big[\sum_{k}c_s(ak)e^{\lambda\bar \probefunction_s(ak)}\Big].
\]
If
$C_s(a)=\emptyset$ all children are empty and the inner expectation is $0$. Otherwise it is $\le1$ by \eqref{eq:star}, and the event $\{C_s(a)\neq\emptyset\}$ is $\Ft_{s-1}$-measurable by Assumption~\ref{ass:meas}. Either way the summand is at most
$U(a)$, so $\Es[M_s]\le M_{s-1}$.
\end{proof}

\begin{proof}[Proof of Corollary~\ref{cor:masterband}]
By Ville's inequality, $\PP(\exists t:M_t\ge1/\alpha)\le\alpha$. On the
complement, fix $t,v$ and let $a:=a_{t+1}(v)$, so $v\in C_{t+1}(a)$ and
$M_t\ge w_t(a)e^{\lambda\tilde D_t(a)}$, all other terms being nonnegative.
Hence $\lambda\tilde D_t(a)\le\ln\frac1{\alpha w_t(a)}$, and
Lemma~\ref{lem:dom} gives
$\lambda D_t(v)\le\lambda\tilde D_t(a)\le\Lcode_t(v)+\ln\frac1\alpha$.
\end{proof}
\newpage

\newpage
\section{GENERIC CHARGE SCHEMES}\label{sec:templates}

Recall from (R4) that $\stay\in\Alp$ is the distinguished symbol emitted by a
cell that does not divide. Fix $a\in\Alp^{s-1}$ with $C_s(a)\neq\emptyset$.
Exactly one of two things happens at step $s$. Either $a$ \emph{does not
divide}, in which case $C_{s+1}(a\stay)=C_s(a)$ and all other children of
$a$ are empty, or $a$ \emph{divides}, in which case
$C_{s+1}(a\stay)=\emptyset$. The event that $a$ divides is
$\cF_s$-measurable, and $\cS_t(v)$ is the set of steps at which the cell of
$v$ divides.

In the threshold refinement of Section~\ref{sec:cdf}, the observation
$X_s$ splits the cell containing it at the point $X_s$, and leaves every other
cell unchanged. After this split, the observed indicator
$v\mapsto\ind\{X_s\le v\}$ no longer varies within any cell. The following
definition isolates this property.

\begin{definition}[Homogeneous refinement]\label{def:homog}
A refinement scheme is \emph{homogeneous} for $(f_s)_{s\ge1}$ if, for every
$s\ge1$, the map $v\mapsto f_s(v)$ is constant on every cell of
$\cC_{s+1}$. For $b\in\Alp^{s}$ with $C_{s+1}(b)\neq\emptyset$, write
$f_s[b]$ for this constant value. If $C_{s+1}(b)=\emptyset$, set
$f_s[b]:=0$.
\end{definition}

In words, a cell that does not divide carries a constant value of $f_s$, and a
cell that divides does so into pieces on which $f_s$ is constant. The coarsest
homogeneous refinement divides each cell into the level sets of $f_s$. Finer
refinements are also allowed.

Homogeneity separates the deviation on a cell into an observed and an
unobserved part. For $v\in C_{s+1}(b)$,
\[
  f_s(v)-\theta_s(v)=f_s[b]-\theta_s(v).
\]
The first term is known once $X_s$ is observed. For the second, the charges
need a lower bound on $\theta_s$ over the cell whose size they can control.
The following assumption provides it, through the smallest value that $f_s$
takes on the cell.
\begin{assumption}[Cell floor]\label{ass:floor}
Fix an $\mathcal{F}_{s-1}$-measurable version $\theta_s(v)$ of $\E_{s-1}[f_s(v)]$ for every $s\ge1$, $v\in \mathcal{V}$.
For every $s\ge1$ and $a\in \Alp^{s-1}$, on $\{C_s(a)\neq\emptyset\}$ the cell minimum
$f^\downarrow_s(a):=\inf_{v\in C_s(a)}f_s(v)$ is $\mathcal{F}_s$-measurable with
$\E_{s-1}|f^\downarrow_s(a)|<\infty$, and almost surely
\[
  p_s(a):=\E_{s-1}\big[f^\downarrow_s(a)\big]\le\theta_s(v)\qquad\text{for all }v\in C_s(a).
\]
\end{assumption}

In our examples the cell minimum is explicit. For the threshold cell
$[l,r)$ it is $\ind\{X_s\le l\}$, whose conditional mean is
$F_s(l)=\inf_{v\in[l,r)}F_s(v)$.

\begin{lemma}[Bernoulli cumulant]\label{lem:bern}
For $\lambda>0$ and $p\in[0,1]$, let $\Lambda_p(\lambda):=\ln(1+p(e^\lambda-1))$
and $K(\lambda,p):=\Lambda_p(\lambda)-\lambda p$.
\begin{enumerate}
\item[(i)] The map $p\mapsto\Lambda_p(\lambda)$ is continuous, increasing and
concave on $[0,1]$.
\item[(ii)] $0\le K(\lambda,p)\le\lambda^2/8$.
\item[(iii)] If $Z\in[0,1]$ and $\E_{s-1}[Z]\le p$, then
$\E_{s-1}[e^{\lambda Z}]\le e^{\Lambda_p(\lambda)}$.
\end{enumerate}
\end{lemma}

\begin{proof}
(i) We have
$\partial_p\Lambda_p=(e^\lambda-1)/(1+p(e^\lambda-1))>0$ and
$\partial_p^2\Lambda_p=-(e^\lambda-1)^2/(1+p(e^\lambda-1))^2<0$.

(ii) If $B\sim\mathrm{Bernoulli}(p)$, then
$K(\lambda,p)=\ln\E e^{\lambda(B-p)}$. This is nonnegative by Jensen's
inequality and at most $\lambda^2/8$ by Hoeffding's lemma.

(iii) By convexity, $e^{\lambda z}\le1+z(e^\lambda-1)$ for $z\in[0,1]$. Taking
conditional expectations gives
$\E_{s-1}e^{\lambda Z}\le1+p(e^\lambda-1)=e^{\Lambda_p(\lambda)}$.
\end{proof}

\subsection{Template A for bounded families}

\begin{proposition}[Template A]\label{prop:A}
Assume that $f_s(v)\in[0,1]$ for all $s$ and $v$, that the refinement scheme
is homogeneous for $(f_s)$, and that Assumptions~\ref{ass:meas}
and~\ref{ass:floor} hold. Fix $\lambda>0$. For $s\ge1$ and $a\in\Alp^{s-1}$,
let $\gamma_s(a)$ be $\cF_{s-1}$-measurable with
$\gamma_s(a)\ge K(\lambda,p_s(a))$, and for $k\neq\stay$ let $\pi_s(ak)\ge0$ be
$\cF_s$-measurable with $\sum_{k\neq\stay}\pi_s(ak)\le1$. Then
\[
  \bar f_s(ak):=f_s[ak]-p_s(a)
\]
(with $\bar f_s(ak):=-\infty$ for empty children) is an envelope, and
\[
  c_s(a\stay):=e^{-\gamma_s(a)},\qquad
  c_s(ak):=\pi_s(ak)\,e^{-\lambda f_s[ak]-\gamma_s(a)}
  \quad(k\neq\stay)
\]
is a charge scheme at $\lambda$.
\end{proposition}

\begin{proof}
Fix $s$ and $a$, and work on $\{C_s(a)\neq\emptyset\}$. Write
$p:=p_s(a)$ and $\gamma:=\gamma_s(a)$.

\emph{Envelope.} For $v\in C_{s+1}(ak)\subseteq C_s(a)$, homogeneity and
Assumption~\ref{ass:floor} give $f_s(v)-\theta_s(v)\le f_s[ak]-p$.

\emph{Charge condition.} Let $Z:=f_s[a\stay]$ if $a$ does not divide, and
$Z:=0$ if it divides. If $a$ does not divide, the only nonempty child is
$a\stay$, and the sum $\sum_kc_s(ak)e^{\lambda\bar f_s(ak)}$ equals $e^{-\gamma}e^{\lambda(Z-p)}$. If $a$
divides, it equals $e^{-\gamma-\lambda p}\sum_{k\neq\stay}\pi_s(ak)\le
e^{-\gamma-\lambda p}$. In both cases, since $\gamma+\lambda p\ge
K(\lambda,p)+\lambda p=\Lambda_p(\lambda)$,
\[
  \sum_kc_s(ak)e^{\lambda\bar f_s(ak)}\le e^{-\Lambda_p(\lambda)}e^{\lambda Z}.
\]
Moreover $Z=f^\downarrow_s(a)\,\ind\{a\text{ does not divide}\}$, since without
a division $f_s\equiv f_s[a\stay]$ on $C_s(a)=C_{s+1}(a\stay)$. Hence
$0\le Z\le f^\downarrow_s(a)$. We obtain $\E_{s-1}[Z]\le p$, and
Lemma~\ref{lem:bern}(iii) gives $\E_{s-1}[e^{\lambda Z}]\le e^{\Lambda_p(\lambda)}$. Hence,
\[
\Es\left[\sum_k c_s(ak)e^{\lambda\bar \probefunction_s(ak)}\right]\le1
\]
\end{proof}

For the path of $v$ and $s\le t$, define
\[
  \pi_s(v):=
  \begin{cases}
    \pi_s(a_{s+1}(v)) & \text{if } s\in\cS_t(v),\\
    1 & \text{if } s\notin\cS_t(v),
  \end{cases}
  \qquad
  \gamma_s(v):=\gamma_s(a_s(v)),
  \qquad
  p_s(v):=p_s(a_s(v))\le\theta_s(v).
\]
Let $\check f_s(v):=f_s(v)$ if the cell of $v$ does not divide at step $s$,
that is, if $k_s(v)=\stay$, and $\check f_s(v):=0$ if it divides. Further, set
\[
  \Pi_t(v):=\sum_{s\in\cS_t(v)}\ln\frac1{\pi_s(v)},\qquad
  \check\mu_t(v):=\frac1t\sum_{s\le t}\check f_s(v),\qquad
  P_t(v):=\frac1t\sum_{s\le t}p_s(v)\le\mu_t(v).
\]
Here $\check\mu_t(v)$ is the empirical mean in which every observation that
divided the cell of $v$ is replaced by $0$.

Given the charge scheme of Proposition~\ref{prop:A}, we obtain the following.

\begin{corollary}\label{cor:boundedsimple}
Let the conditions of Proposition~\ref{prop:A} hold. Then
\[
  \Dt_t(a_{t+1}(v))=t\bigl(\widehat\mu_t(v)-P_t(v)\bigr).
\]
Set
\[
  \ell_t(v)=\sum_{s\le t}\gamma_s(v)+\Pi_t(v)+\lambda\sum_{s\in\cS_t(v)}f_s(v).
\]
Let $\alpha\in(0,1)$. With probability at least $1-\alpha$, for all $t$ and
$v$,
\[
  \lambda t\bigl(\widehat\mu_t(v)-P_t(v)\bigr)\le\ell_t(v)+\ln\tfrac1\alpha.
\]
\end{corollary}

\begin{proof}
\emph{Envelope.} Since $v\in C_{s+1}(a_{s+1}(v))$,
$\bar f_s(a_{s+1}(v))=f_s[a_{s+1}(v)]-p_s(a_s(v))=f_s(v)-p_s(v)$. Summing
over $s\le t$ gives the identity for $\Dt_t$.

\emph{Path cost.} If $s\notin\cS_t(v)$, then $a_{s+1}(v)=a_s(v)\stay$ and the
charge is $e^{-\gamma_s(v)}$. If $s\in\cS_t(v)$, the charge is
$\pi_s(v)e^{-\lambda f_s[a_{s+1}(v)]-\gamma_s(v)}$, and
$f_s[a_{s+1}(v)]=f_s(v)$ because $v\in C_{s+1}(a_{s+1}(v))$. Summing
$-\ln c_s$ over $s\le t$ gives the formula for $\ell_t(v)$.

\emph{High-probability bound.} This is Corollary~\ref{cor:masterband}.
\end{proof}

We can now apply the previous result with explicit choices of $\gamma$.

\begin{corollary}[Template A bands]\label{cor:A}
Let the conditions of Proposition~\ref{prop:A} hold, and let
$\alpha\in(0,1)$.
\begin{enumerate}
\item[(i)] \emph{Hoeffding form.} With $\gamma\equiv\lambda^2/8$, with
probability at least $1-\alpha$, for all $t$ and $v$,
\[
  \mu_t(v)\ \ge\ P_t(v)\ \ge\ \check\mu_t(v)-\frac\lambda8-\frac{\Pi_t(v)+\ln\frac1\alpha}{\lambda t}.
\]
\item[(ii)] \emph{Exact form.} With $\gamma_s(a)=K(\lambda,p_s(a))$ for all $a$,
with probability at least $1-\alpha$, for all $t$ and $v$,
\[
  \mu_t(v)\ \ge\ P_t(v)\ \ge\ \frac{\exp\bigl\{\lambda\check\mu_t(v)-(\Pi_t(v)+\ln\frac1\alpha)/t\bigr\}-1}{e^\lambda-1}.
\]
\end{enumerate}
\end{corollary}

\begin{proof}
Both bounds follow from Corollary~\ref{cor:boundedsimple}.

(i) With $\gamma\equiv\lambda^2/8$ we have
$\ell_t(v)=t\lambda^2/8+\Pi_t(v)+\lambda\sum_{s\in\cS_t(v)}f_s(v)$. Hence, with
probability at least $1-\alpha$, for all $t$ and $v$,
\begin{align*}
  \mu_t(v)\ \ge\ P_t(v)
  &\ge\widehat\mu_t(v)-\frac{\ell_t(v)+\ln\frac1\alpha}{\lambda t}
   =\widehat\mu_t(v)-\frac1t\sum_{s\in\cS_t(v)}f_s(v)-\frac\lambda8-\frac{\Pi_t(v)+\ln\frac1\alpha}{\lambda t}\\
  &=\check\mu_t(v)-\frac\lambda8-\frac{\Pi_t(v)+\ln\frac1\alpha}{\lambda t}.
\end{align*}

(ii) With $\gamma_s(v)=K(\lambda,p_s(v))$ we have
$\gamma_s(v)+\lambda p_s(v)=\Lambda_{p_s(v)}(\lambda)$. Using
Corollary~\ref{cor:boundedsimple}, we obtain, with probability at least $1-\alpha$, for all $t$ and $v$,
\begin{align*}
  \lambda t\,\check\mu_t(v)
  &=\lambda\Bigl(t\widehat\mu_t(v)-\sum_{s\in\cS_t(v)}f_s(v)\Bigr)
   \le\lambda tP_t(v)-\lambda\sum_{s\in\cS_t(v)}f_s(v)+\ell_t(v)+\ln\tfrac1\alpha\\
  &=\sum_{s\le t}\bigl(\lambda p_s(v)+\gamma_s(v)\bigr)+\Pi_t(v)+\ln\tfrac1\alpha
   =\sum_{s\le t}\Lambda_{p_s(v)}(\lambda)+\Pi_t(v)+\ln\tfrac1\alpha\\
  &\le t\,\Lambda_{P_t(v)}(\lambda)+\Pi_t(v)+\ln\tfrac1\alpha .
\end{align*}
The last step is Jensen's inequality for the concave map
$p\mapsto\Lambda_p(\lambda)$. Since
$\Lambda_P(\lambda)=\ln(1+P(e^\lambda-1))$ is increasing in $P$, this
rearranges to the claim.
\end{proof}

In both forms, every observation that divided the cell of $v$ is counted as
$0$ in $\check\mu_t(v)$. Applying the corollary to $1-f_s$, which is
homogeneous for the same refinement, counts it as $1$ on the upper side.
Moreover, $v$ pays the code length $\Pi_t(v)$ of its path.
\begin{remark}[Choice of split weights]\label{rem:weights}
Templates~A and~B leave the weights $\pi_s(ak)$ of the children of a divided
cell free, subject only to summing to at most one, and they may depend on the
past and on the current observation. Along the path of $v$ they enter only
through $\Pi_t(v)=\sum_{s\in\cS_t(v)}\ln(1/\pi_s(v))$, the length of a code for
the sequence of moves of $v$. Any sequential probability assignment on these
moves therefore gives a valid band. Examples are a fixed tilt towards one side,
or the Krichevsky--Trofimov weights \citep{KT}, which adapt to the moves seen
so far. Such weights are cheaper for probes whose moves are mostly to one side,
for instance in the tails, and more expensive for probes with balanced moves.
Throughout, we use equal weights, $\frac12$ for each child of a split threshold
cell and $2^{-m}$ for each piece of a box cut in $m$ coordinates. 
\end{remark}

\subsection{Template A without homogeneity}\label{sec:templateA-nonhom}

Homogeneity is not needed for Template~A. In Proposition~\ref{prop:A} it serves only to make the value
$f_s[ak]$ of the current observation on each child known once $X_s$ is observed, which gives the
exact envelope $f_s[ak]-p_s(a)$. Without homogeneity, we bound $f_s$ on a child by its maximum there,
and we compare the current observation with its minimum over the parent cell. The price is the
oscillation of $f_s$ over the cell of $v$ at the time $X_s$ arrives
(Remark~\ref{rem:oscillation}). In return, the refinement may divide cells at any
$\mathcal{F}_s$-measurable points, or deterministically.

For a string $b$ with $|b|\in\{s-1,s\}$ let $C(b):=C_{|b|+1}(b)$ be the cell it indexes, so that
$C(a)=C_s(a)$ for $a\in\Alp^{s-1}$ and $C(ak)=C_{s+1}(ak)$ for its children. On
$\{C(b)\neq\emptyset\}$ write
\[
  f^\downarrow_s(b):=\inf_{v\in C(b)}f_s(v),\qquad f^\uparrow_s(b):=\sup_{v\in C(b)}f_s(v).
\]
For $a\in\Alp^{s-1}$, $f^\downarrow_s(a)$ is the cell minimum of Assumption~\ref{ass:floor}. If $a$ does not
divide, then $C(a\star)=C(a)$, so $f^\downarrow_s(a\star)=f^\downarrow_s(a)$ and
$f^\uparrow_s(a\star)=f^\uparrow_s(a)$.

\begin{proposition}[Template A$'$]\label{prop:templateAprime}
Assume that $f_s(v)\in[0,1]$ for all $s$ and $v$, that Assumptions~\ref{ass:meas} and~\ref{ass:floor} hold, and that
$f^\uparrow_s(b)$ is $\mathcal{F}_s$-measurable for every nonempty child $b\in\Alp^s$. Fix $\lambda>0$.
For $s\ge1$ and $a\in\Alp^{s-1}$, let $\gamma_s(a)$ be $\mathcal{F}_{s-1}$-measurable with $\gamma_s(a)\ge K(\lambda,p_s(a))$. For $k\neq\star$, let $\pi_s(ak)\ge0$ be $\mathcal{F}_s$-measurable with $\sum_{k\neq\star}\pi_s(ak)\le1$, and put $\pi_s(a\star):=1$.
Then
\[
  \bar f_s(ak):=f^\uparrow_s(ak)-p_s(a)
\]
is an envelope, and
\[
  c_s(ak):=\pi_s(ak)\exp\bigl\{-\lambda\bigl(f^\uparrow_s(ak)-f^\downarrow_s(a)\bigr)-\gamma_s(a)\bigr\}
\]
is a charge scheme at $\lambda$. For empty children, set $\bar f_s(ak):=-\infty$ and $c_s(ak):=0$.
\end{proposition}

\begin{proof}
Fix $s$ and $a$, work on $\{C_s(a)\neq\emptyset\}$, and write $p:=p_s(a)$ and $\gamma:=\gamma_s(a)$.

\emph{Envelope.} The variable $\bar f_s(ak)$ is $\mathcal{F}_s$-measurable. By Assumption~\ref{ass:floor}, almost
surely $\theta_s(v)\ge p$ for all $v\in C_s(a)$. Hence, almost surely, for every nonempty child $ak$
and every $v\in C_{s+1}(ak)\subseteq C_s(a)$,
\[
  f_s(v)-\theta_s(v)\le f^\uparrow_s(ak)-p=\bar f_s(ak).
\]

\emph{Charge condition.} For every nonempty child,
\[
  c_s(ak)\,e^{\lambda\bar f_s(ak)}=\pi_s(ak)\,e^{\lambda(f^\downarrow_s(a)-p)-\gamma}.
\]
Only nonempty children contribute, and their weights sum to at most one. If $a$ does not divide,
$a\star$ is the only nonempty child and $\pi_s(a\star)=1$; if $a$ divides, $C_{s+1}(a\star)=\emptyset$.
Since $\gamma+\lambda p\ge K(\lambda,p)+\lambda p=\Lambda_p(\lambda)$,
\[
  \sum_k c_s(ak)\,e^{\lambda\bar f_s(ak)}\le e^{\lambda f^\downarrow_s(a)-\Lambda_p(\lambda)}.
\]
The variable $f^\downarrow_s(a)\in[0,1]$ is $\mathcal{F}_s$-measurable with
$\mathbb E_{s-1}[f^\downarrow_s(a)]=p$, so Lemma~\ref{lem:bern}(iii) gives
$\mathbb E_{s-1}\bigl[e^{\lambda f^\downarrow_s(a)}\bigr]\le e^{\Lambda_p(\lambda)}$.
\end{proof}

The two choices of $\gamma$ as before give the explicit charges
\begin{align*}
  c_s(ak)&=\pi_s(ak)\,e^{-\lambda(f^\uparrow_s(ak)-f^\downarrow_s(a))-\lambda^2/8}
  &&(\gamma\equiv\lambda^2/8),\\
  c_s(ak)&=\pi_s(ak)\,e^{-\lambda(f^\uparrow_s(ak)-f^\downarrow_s(a))}\,
  \frac{e^{\lambda p_s(a)}}{1+p_s(a)(e^\lambda-1)}
  &&(\gamma_s(a)=K(\lambda,p_s(a))).
\end{align*}
Under homogeneity, $f_s$ is constant on each child, so $f^\uparrow_s(ak)=f_s[ak]$ and the envelope is
that of Proposition~\ref{prop:A}. Comparing the charges in the two cases, we note the following. If $a$ does not divide, then $f^\downarrow_s(a)=f_s[a\star]$, so $c_s(a\star)=e^{-\gamma_s(a)}$, as in Proposition~\ref{prop:A}. If $a$ divides, then $f^\downarrow_s(a)=\min_{k:\,C_{s+1}(ak)\neq\emptyset}f_s[ak]$, and $c_s(ak)$
  is $e^{\lambda f^\downarrow_s(a)}\ge1$ times the charge of Proposition~\ref{prop:A}.
In other words, Template~A counts a dividing observation as $0$, while Template~A$'$ counts it as
$f^\downarrow_s(a)$. For thresholds, boxes and their weighted versions some child has value $0$, so
$f^\downarrow_s(a)=0$ at every division and the two templates coincide.

For the path of $v$, let $p_s(v)$, $\gamma_s(v)$, $\pi_s(v)$, $\Pi_t(v)$ and $P_t(v)$ be as before
Corollary~\ref{cor:boundedsimple}, and put
\[
  \mu^\downarrow_t(v):=\frac1t\sum_{s\le t}f^\downarrow_s(a_s(v)),\qquad
  \mu^\uparrow_t(v):=\frac1t\sum_{s\le t}f^\uparrow_s(a_s(v)).
\]
Thus $\mu^\downarrow_t(v)$ counts each observation by its minimum over the cell of $v$ at the time it
arrives. 
Under homogeneity, $f_s$ is constant on a cell that does not divide, so
$f^\downarrow_s(a_s(v))=f_s(v)$ whenever the cell of $v$ does not divide at step $s$. When it divides, $f^\downarrow_s(a_s(v))\ge0$. Hence $\mu^\downarrow_t(v)\ge\check\mu_t(v)$, with equality whenever $f^\downarrow_s(a)=0$ at every division.

\begin{corollary}[Template A$'$ bands]\label{cor:templateAprime}
Let the conditions of Proposition~\ref{prop:templateAprime} hold, and let $\alpha\in(0,1)$.
\begin{enumerate}
\item[(i)] With $\gamma\equiv\lambda^2/8$, with probability at least $1-\alpha$, for all $t$ and $v$,
\[
  \mu_t(v)\ \ge\ P_t(v)\ \ge\ \mu^\downarrow_t(v)-\frac\lambda8-\frac{\Pi_t(v)+\ln\frac1\alpha}{\lambda t}.
\]
\item[(ii)] With $\gamma_s(a)=K(\lambda,p_s(a))$, with probability at least $1-\alpha$, for all $t$ and $v$,
\[
  \mu_t(v)\ \ge\ P_t(v)\ \ge\
  \frac{\exp\{\lambda\mu^\downarrow_t(v)-(\Pi_t(v)+\ln\frac1\alpha)/t\}-1}{e^\lambda-1}.
\]
\end{enumerate}
\end{corollary}

\begin{proof}
\emph{Path cost.} We have $a_{s+1}(v)=a_s(v)k_s(v)$. If $s\notin S_t(v)$, then $k_s(v)=\star$ and the
weight is $\pi_s(a\star)=1$; otherwise it is $\pi_s(v)$. Summing $-\ln c_s(a_{s+1}(v))$ over $s\le t$,
\[
  \ell_t(v)=\Pi_t(v)+\sum_{s\le t}\Bigl[\lambda\bigl(f^\uparrow_s(a_{s+1}(v))-f^\downarrow_s(a_s(v))\bigr)
  +\gamma_s(v)\Bigr],
\]
and
\[
  \tilde D_t(a_{t+1}(v))=\sum_{s\le t}\bigl(f^\uparrow_s(a_{s+1}(v))-p_s(v)\bigr).
\]

\emph{High-probability bound.} By Corollary~5.8, with probability at least $1-\alpha$,
$\lambda\tilde D_t(a_{t+1}(v))\le\ell_t(v)+\ln\frac1\alpha$ for all $t$ and $v$. The child maxima
cancel, and
\begin{equation}\label{eq:pathAprime}
  \lambda t\bigl(\mu^\downarrow_t(v)-P_t(v)\bigr)\le\Pi_t(v)+\sum_{s\le t}\gamma_s(v)+\ln\tfrac1\alpha .
\end{equation}
Moreover, almost surely $P_t(v)\le\mu_t(v)$ for all $t$ and $v$. Indeed $v\in C_s(a_s(v))$, so
Assumption~\ref{ass:floor}, applied to the countably many pairs $(s,a)$, gives $p_s(v)\le\theta_s(v)$.

(i) With $\gamma\equiv\lambda^2/8$, divide \eqref{eq:pathAprime} by $\lambda t$.

(ii) With $\gamma_s(v)=K(\lambda,p_s(v))$ we have $\lambda p_s(v)+\gamma_s(v)=\Lambda_{p_s(v)}(\lambda)$.
Then \eqref{eq:pathAprime} and concavity (Lemma~\ref{lem:bern}(i)) give
\[
  \lambda t\mu^\downarrow_t(v)\le\sum_{s\le t}\Lambda_{p_s(v)}(\lambda)+\Pi_t(v)+\ln\tfrac1\alpha
  \le t\Lambda_{P_t(v)}(\lambda)+\Pi_t(v)+\ln\tfrac1\alpha.
\]
Rearrange as in Corollary~\ref{cor:A}(ii).
\end{proof}

\begin{remark}[Two-sided bands: oscillation plus code length]\label{rem:oscillation}
Suppose Assumption~\ref{ass:floor} also holds for $1-f_s$. That is, $f^\uparrow_s(a)$ is $\mathcal{F}_s$-measurable,
and almost surely $\mathbb E_{s-1}[f^\uparrow_s(a)]\ge\theta_s(v)$ for all $v\in C_s(a)$. Suppose further
that $f^\downarrow_s(b)$ is $\mathcal{F}_s$-measurable for every nonempty child $b\in\Alp^s$. Then
Corollary~\ref{cor:templateAprime}(i) applies to $1-f_s$. Its cell minimum along the path of $v$ is
$1-f^\uparrow_s(a_s(v))$, so it yields
\[
  \mu_t(v)\le\mu^\uparrow_t(v)+\frac\lambda8+\frac{\Pi'_t(v)+\ln\frac1\alpha}{\lambda t},
\]
where $\Pi'_t$ is the code length of the split weights used for $1-f_s$.

Take level $\alpha/2$ on each side and the same split weights. Then the band has width
\[
  \frac1t\sum_{s\le t}\operatorname{osc}_s(v)+\frac\lambda4+\frac{2\bigl(\Pi_t(v)+\ln\frac2\alpha\bigr)}{\lambda t},
  \qquad
  \operatorname{osc}_s(v):=f^\uparrow_s(a_s(v))-f^\downarrow_s(a_s(v)).
\]
The first term is the average oscillation of the current observation over the cell of $v$ at the time
it arrives, whereas the rest is a code length term. A refinement is a tradeoff between the two terms. Small cells reduce the
oscillation, and few divisions reduce the code length. A homogeneous refinement removes the
oscillation except at divisions. For thresholds, $\operatorname{osc}_s(v)=\mathbf 1\{s\in \cS_t(v)\}$,
which is the term $N_t(v)/t$ in the width of Theorem~\ref{thm:cdfbandtrivial}.
\end{remark}

\subsection{Template B for nonnegative families}

For nonnegative families that may be unbounded, Hoeffding's lemma is not
available. We replace it by a one-step betting argument \citep{betting}, in the form of \citet{mineiro2022lowerconfidencesequencechanging}, applied cell by cell, which uses only $f_s\ge0$ and does not require
homogeneity.

\begin{proposition}[Template B]\label{prop:B}
Assume that $f_s(v)\ge0$ for all $s$ and $v$, and that
Assumption~\ref{ass:meas} holds. For each nonempty child $b\in\Alp^s$, assume
that $f^\uparrow_s(b):=\sup_{C_{s+1}(b)}f_s$ is $\cF_s$-measurable and finite.
Fix $\lambda>0$. For $a\in\Alp^{s-1}$, let $\Ah_s(a)$ be
$\cF_{s-1}$-measurable with $0\le\lambda\Ah_s(a)<1$, and let the \emph{floor}
$y_s(a)\ge0$ be $\cF_s$-measurable with
\[
  q_s(a):=\E_{s-1}\bigl[y_s(a)\bigr]\le\theta_s(v)
  \qquad\text{for all }v\in C_s(a),\ \text{a.s.\ on }\{C_s(a)\neq\emptyset\}.
\]
Let $\pi_s(ak)\ge0$ be $\cF_s$-measurable with $\sum_{k\in\Alp}\pi_s(ak)\le1$.
Then $\bar f_s(ak):=f^\uparrow_s(ak)-q_s(a)$ is an envelope, and
\[
  c_s(ak):=\pi_s(ak)\,\bigl(1+\lambda(y_s(a)-\Ah_s(a))\bigr)\,
  e^{-\lambda(f^\uparrow_s(ak)-\Ah_s(a))}
\]
is a charge scheme at $\lambda$. For empty children, set
$\bar f_s(ak):=-\infty$ and $c_s(ak):=0$.
\end{proposition}

\begin{proof}
Fix $a$, work on $\{C_s(a)\neq\emptyset\}$, and write $q:=q_s(a)$,
$\Ah:=\Ah_s(a)$ and $y:=y_s(a)$.

\emph{Envelope.} For $v\in C_{s+1}(ak)\subseteq C_s(a)$ we have
$f_s(v)-\theta_s(v)\le f^\uparrow_s(ak)-q$.

\emph{Charge condition.} Since $y\ge0$, we have
$1+\lambda(y-\Ah)\ge1-\lambda\Ah>0$, so every charge is nonnegative, and

\begin{align*}
    \sum_kc_s(ak)e^{\lambda\bar f_s(ak)}&=\sum_k\pi_s(ak)\,\left(1+\lambda\left(y-\Ah\right)\right)
  e^{-\lambda\left(f^\uparrow_s(ak)-\Ah\right)}e^{\lambda\left(f^\uparrow_s(ak)-q\right)}\\
    &=\sum_k\pi_s(ak)\,\left(1+\lambda\left(y-\Ah\right)\right)\,
  e^{\lambda(\Ah-q)}\le\bigl(1+\lambda(y-\Ah)\bigr)e^{\lambda(\Ah-q)} 
\end{align*}
The factor $e^{\lambda(\Ah-q)}$ is $\cF_{s-1}$-measurable, and
$\E_{s-1}[1+\lambda(y-\Ah)]=1+\lambda(q-\Ah)$. Hence the conditional
expectation is at most $\bigl(1+\lambda(q-\Ah)\bigr)e^{\lambda(\Ah-q)}\le1$,
using $1+x\le e^x$.
\end{proof}

\begin{remark}
The variable $\Ah_s(a)$ is a \emph{centering}, a predictable guess of the
values that $f_s$ takes on the cell.
\end{remark}

Under Assumption~\ref{ass:floor}, the cell minimum $f^\downarrow_s(a)$ is
always a floor. A larger floor gives a smaller path cost. Such floors can be
available when $f_s$ is not homogeneous (Appendix~\ref{sec:cens}).

For the path of $v$, let
\[
  \pi_s(v):=\pi_s(a_{s+1}(v)),\quad
  y_s(v):=y_s(a_s(v)),\quad
  \Ah_s(v):=\Ah_s(a_s(v)),\quad
  f^\uparrow_s(v):=f^\uparrow_s(a_{s+1}(v))\ge f_s(v).
\]

Given the charge scheme of Proposition~\ref{prop:B}, we obtain the following.

\begin{corollary}\label{cor:B}
Let the conditions of Proposition~\ref{prop:B} hold. Then
\[
  \ell_t(v)=\sum_{s\le t}\left[\ln\frac1{\pi_s(v)}-\ln\bigl(1+\lambda(y_s(v)-\Ah_s(v))\bigr)+\lambda\bigl(f^\uparrow_s(v)-\Ah_s(v)\bigr)\right].
\]
Let $\alpha\in(0,1)$. With probability at least $1-\alpha$, for all $t$ and
$v$, $\lambda\Dt_t(a_{t+1}(v))\le\ell_t(v)+\ln\frac1\alpha$, and hence
\[
  \mu_t(v)\ \ge\ \widehat\mu_t(v)-\frac{\ell_t(v)+\ln\frac1\alpha}{\lambda t}.
\]
\end{corollary}

\begin{proof}
Read off $-\ln c_s(a_{s+1}(v))$ for each $s\le t$, and apply
Corollary~\ref{cor:masterband} and Lemma~\ref{lem:DleqDt}.
\end{proof}

For homogeneous families the band takes a simpler form. For $\lambda>0$,
$a\ge0$ with $\lambda a<1$, and $x\ge0$, put
\[
  \psi_\lambda(x;a):=a+\frac{\ln\bigl(1+\lambda(x-a)\bigr)}{\lambda}
  =x-\frac{\eta\bigl(\lambda(x-a)\bigr)}{\lambda},
  \qquad \eta(y):=y-\ln(1+y).
\]

\begin{corollary}[Template B for homogeneous families]\label{cor:Bpsi}
Let the conditions of Proposition~\ref{prop:B} hold, and assume in addition
that the refinement scheme is homogeneous for $(f_s)$ and that
Assumption~\ref{ass:floor} holds. Choose $y_s(a):=f_s[a\stay]$, the value of
$f_s$ on the cell, if $a$ does not divide, and any $\cF_s$-measurable
$y_s(a)\in[0,f^\downarrow_s(a)]$ if $a$ divides. The choice $y_s(a)=0$ is
always admissible. Take $\pi_s(a\stay):=1$ when $a$ does not divide. For the
path of $v$, let $\Pi_t(v)$ be as in Corollary~\ref{cor:boundedsimple}, and
let $\check f_s(v):=f_s(v)$ if the cell of $v$ does not divide at step $s$,
and $\check f_s(v):=y_s(a_s(v))$ if it divides. With the floor $0$ at
divisions, this is the $\check f_s$ of Template~A. With probability at least
$1-\alpha$, for all $t$ and $v$,
\[
  \mu_t(v)\ \ge\ \frac1t\sum_{s\le t}\psi_\lambda\bigl(\check f_s(v);\Ah_s(v)\bigr)-\frac{\Pi_t(v)+\ln\frac1\alpha}{\lambda t}.
\]
\end{corollary}

\begin{proof}
\emph{Admissibility of $y$.} If $a$ does not divide, then $f_s\equiv
f_s[a\stay]$ on $C_s(a)=C_{s+1}(a\stay)$, so $y_s(a)=f^\downarrow_s(a)$. In
both cases $0\le y_s(a)\le f^\downarrow_s(a)$, so $q_s(a)\le p_s(a)$, and
Assumption~\ref{ass:floor} shows that $y_s(a)$ is a floor. It is
$\cF_s$-measurable because the event that $a$ divides is.\\
\emph{Path cost.} By homogeneity, $f^\uparrow_s(v)=f_s(v)$. At non-dividing
steps, $y_s(a_s(v))=f_s(v)=\check f_s(v)$, and at dividing steps
$y_s(a_s(v))=\check f_s(v)$ by definition. Moreover $\pi_s(v)=1$ at
non-dividing steps, so $\sum_{s\le t}\ln(1/\pi_s(v))=\Pi_t(v)$.
Corollary~\ref{cor:B} therefore gives
\[
  \ell_t(v)=\Pi_t(v)+\lambda\sum_{s\le t}f_s(v)
  -\sum_{s\le t}\Bigl[\lambda\Ah_s(v)+\ln\bigl(1+\lambda(\check f_s(v)-\Ah_s(v))\bigr)\Bigr].
\]
Each logarithm is well defined, since $\check f_s(v)\ge0$ and
$\lambda\Ah_s(v)<1$.

\emph{Conclusion.} By Corollary~\ref{cor:B}, with probability at least
$1-\alpha$, for all $t$ and $v$,
$\lambda t\bigl(\widehat\mu_t(v)-\mu_t(v)\bigr)\le\ell_t(v)+\ln\frac1\alpha$.
Since $\lambda\sum_{s\le t}f_s(v)=\lambda t\,\widehat\mu_t(v)$, this term
cancels with the left-hand side, and dividing by $\lambda t$ gives
\[
  \mu_t(v)\ \ge\ \frac1t\sum_{s\le t}\left[\Ah_s(v)+\frac{\ln\bigl(1+\lambda(\check f_s(v)-\Ah_s(v))\bigr)}{\lambda}\right]
  -\frac{\Pi_t(v)+\ln\frac1\alpha}{\lambda t}.
\]
The bracket is $\psi_\lambda\bigl(\check f_s(v);\Ah_s(v)\bigr)$.
\end{proof}

\newpage

\section{THE HALF-OPEN THRESHOLD REFINEMENT}\label{sec:threshold}
% =====================================================================

We start by showing how to interpret Definition~\ref{def:cellsR} as a
refinement scheme according to Definition~\ref{def:scheme}. For the intuition
behind this construction we refer to Figure~\ref{fig:tree}.

After a split, the cell of $v$ is $[l_s(v),X_s)$ for a left split and
$[X_s,r_s(v))$ for a right split. Otherwise it is unchanged. In the language
of Section~\ref{sec:refinementschemes}, the alphabet is $\Alp:=\{\sL,\sR,\stay\}$, the
symbol $k_s(v)$ is $\sL$, $\sR$ or $\stay$ according to whether $v$ has a left
split, a right split or no split at time $s$, and
$a_{s+1}(v):=a_s(v)k_s(v)$ with $a_1:=\emptyset$. Recall that $\cS_t(v)$ is the set of
split times. We write $\cR_t(v)$ and $\cL_t(v)$ for the sets of right and
left split times.

An observation that repeats an earlier value, or that falls outside the cell
of $v$, leaves the cell of $v$ unchanged. So $N_t(v)$ counts only the
observations that land strictly inside the current cell of $v$. Equivalently,
$N_t(v)$ is the number of comparisons made when searching for $v$ in the binary
search tree of the distinct values of $X_{1:t}$, inserted in order of first
appearance, where the search moves right at a node $x$ if $x\le v$.

\begin{lemma}[Structure]\label{lem:threshold}
\begin{enumerate}
\item[(i)] Definition~\ref{def:cellsR} defines a refinement scheme. Its
cells at time $s$ are the half-open intervals between consecutive distinct
values among $-\infty,X_1,\dots,X_{s-1},+\infty$.
\item[(ii)] A cell $[l,r)$ is split at time $s$ if and only if
$X_s\in(l,r)$, and its children are then $[l,X_s)$ and $[X_s,r)$.
\item[(iii)] $X_s$ lies in the interior of no cell of $\cC_{s+1}$, and
$\inf C_{s+1}(v)\le v<\sup C_{s+1}(v)$ for all $v$.
\end{enumerate}
\end{lemma}

\begin{proof}
If $X_s\notin(l_s(v),r_s(v))$, then $l_{s+1}(v)=l_s(v)$, $r_{s+1}(v)=r_s(v)$, and the cell of $v$ is
unchanged. If $X_s\in(l_s(v),r_s(v))$, then $X_s$ becomes the new
left or right endpoint, according to the side of $v$. Hence, by induction on
$s$, the addresses $a_s(v)$ and the cells $C_s(v)$ determine each other,
which gives (i) and (ii), including condition (R4). The event
$\{a_s(v)=a\}$ is determined by finitely many comparisons among
$v,X_1,\dots,X_{s-1}$, which gives (R3). Finally, for every $v$,
$l_{s+1}(v)\ge X_s$ if $X_s\le v$ and $r_{s+1}(v)\le X_s$ if $X_s>v$, so
$X_s\notin(l_{s+1}(v),r_{s+1}(v))$. Moreover
$v\in C_{s+1}(v)=[l_{s+1}(v),r_{s+1}(v))$. This gives (iii).
\end{proof}

Lemma~\ref{lem:threshold}(iii) is what makes the half-open cells convenient.
On a cell $[l,r)$ of $\cC_{s+1}$, the new observation satisfies either
$X_s\le l$ or $X_s\ge r$. Whether $X_s\le v$ can therefore be read off at the
left endpoint, and whether $X_s>v$ at the right endpoint. The same endpoints
give the cell minima. We state this for weighted indicators, the unweighted
case being $W_s\equiv1$.

\begin{lemma}[Weighted threshold families]\label{lem:families}
Let $W_s\ge0$ be $\cF_s$-measurable with $\E_{s-1}W_s<\infty$. Put
\begin{align*}
  F^W_s(x)&:=\E_{s-1}[W_s\ind\{X_s\le x\}],\\
  H^>_s(x)&:=\E_{s-1}[W_s\ind\{X_s>x\}],\qquad
  H_s(x):=\E_{s-1}[W_s\ind\{X_s\ge x\}],
\end{align*}
with $F^W_s(-\infty):=0$ and $H_s(+\infty):=0$. Consider the two families
\[
  f_s(v):=W_s\ind\{X_s\le v\},\qquad
  g_s(v):=W_s\ind\{X_s>v\},
\]
whose conditional means are $\E_{s-1}[f_s(v)]=F^W_s(v)$ and
$\E_{s-1}[g_s(v)]=H^>_s(v)$.
\begin{enumerate}
\item[(i)] \emph{Homogeneity.} The refinement scheme of Definition~\ref{def:cellsR} is homogeneous for both
families. On a cell $[l,r)$ that is not split, $f_s\equiv W_s\ind\{X_s\le l\}$
and $g_s\equiv W_s\ind\{X_s\ge r\}$. On the children of a split cell,
$f_s=0$ on $\sL$ and $f_s=W_s$ on $\sR$, while $g_s=W_s$ on $\sL$ and
$g_s=0$ on $\sR$.
\item[(ii)] \emph{Cell floor.} On a cell $[l,r)$ of $\cC_s$, the cell
minima are $W_s\ind\{X_s\le l\}$ for $f$ and $W_s\ind\{X_s\ge r\}$ for $g$.
Their conditional means are $F^W_s(l)=\inf_{[l,r)}F^W_s$ and
$H_s(r)=\inf_{[l,r)}H^>_s$, and Assumption~\ref{ass:floor} holds for both
families.
\end{enumerate}
\end{lemma}

\begin{proof}
Fix a cell $C=[l,r)$ of $\cC_{s+1}$. By Lemma~\ref{lem:threshold}(iii),
either $X_s\le l$ or $X_s\ge r$. In the first case $X_s\le v$ for all
$v\in C$, so $f_s\equiv W_s$ and $g_s\equiv0$ on $C$. In the second case
$X_s>v$ for all $v\in C$, so $f_s\equiv0$ and $g_s\equiv W_s$. This proves
homogeneity and the values in (i).

For (ii), let $v\in[l,r)$. Then $\ind\{X_s\le l\}\le\ind\{X_s\le v\}$,
with equality at $v=l$. Moreover $\ind\{X_s\ge r\}\le\ind\{X_s>v\}$, with
equality for $v\in[\max(l,X_s),r)$ when $X_s<r$. This gives the cell minima
and Assumption~\ref{ass:floor}. The function $F^W_s$ is nondecreasing, so its
infimum over the cell is $F^W_s(l)$. The function $H^>_s$ is
nonincreasing, so its infimum over the cell is its left limit at $r$, which is
$H_s(r)$ by dominated convergence, with dominating variable $W_s$. On an
unbounded cell the corresponding minimum and infimum are $0$.
\end{proof}

\newpage
% =====================================================================

\section{MISSING PROOFS \S\ref{sec:cdf}}\label{app:cdf}
% =====================================================================

Let $F_s(x):=\PP_{s-1}(X_s\le x)$ and put
\[
  \Fh_t(v):=\frac1t\sum_{s\le t}\ind\{X_s\le v\},\qquad
  \Fb_t(v):=\frac1t\sum_{s\le t}F_s(v).
\]
Throughout this section we use the two families $f_s(v)=\ind\{X_s\le v\}$ and
$g_s(v)=\ind\{X_s>v\}=1-f_s(v)$, whose conditional means are $F_s(v)$ and
$1-F_s(v)$. These are the families of
Lemma~\ref{lem:families} with $W_s\equiv1$. On the half-open refinement, the
hypotheses of Template~A hold for both. Assumption~\ref{ass:meas} holds by
Lemma~\ref{lem:threshold}(i), and homogeneity and Assumption~\ref{ass:floor}
hold by Lemma~\ref{lem:families}. The floors are $p_s(a)=F_s(l)$ for $f$
and $1-F_s(r^-)$ for $g$, where $[l,r)=C_s(a)$. Every division has
exactly two nonempty children.

\cdfbandtrivial*

\begin{proof}[Proof of Theorem~\ref{thm:cdfbandtrivial}]
\emph{Lower bound.} Apply Corollary~\ref{cor:A}(i) to $f$ at level
$\alpha/2$, with $\gamma\equiv\lambda^2/8$ and weights
$\pi(a\sL)=\pi(a\sR)=\frac12$. Then $\Pi_t(v)=N_t(v)\ln2$. Moreover
$f_s(v)=1$ exactly at the $\sR$-moves, so
$\check\mu_t(v)=\Fh_t(v)-\NR_t(v)/t$. The charges are
$c(a\stay)=e^{-\lambda^2/8}$, $c(a\sL)=\frac12e^{-\lambda^2/8}$ and
$c(a\sR)=\frac12e^{-\lambda^2/8-\lambda}$.

\emph{Upper bound.} Apply Corollary~\ref{cor:A}(i) in the same way to
$g_s(v)=1-f_s(v)=\ind\{X_s>v\}$, whose conditional mean is $1-F_s(v)$. Now
$g_s(v)=1$ exactly at the $\sL$-moves, so
$\check\mu^g_t(v)=1-\Fh_t(v)-\NL_t(v)/t$. This gives
$1-\Fb_t(v)\ge1-\Fh_t(v)-\NL_t(v)/t-\rho_t(N_t(v))$.

A union bound over the two events completes the proof.
\end{proof}

\begin{proof}[Proof of Corollary~\ref{cor:mixture}]
\emph{Coverage.} For each $j\ge0$, Theorem~\ref{thm:cdfbandtrivial} with
$(\lambda_j,\alpha_j)$ fails with probability at most $\alpha_j$. Since
$\sum_{j\ge0}\alpha_j=\alpha\sum_{j\ge0}\bigl(\frac1{j+1}-\frac1{j+2}\bigr)=\alpha$,
a union bound shows that, with probability at least $1-\alpha$, all these
bands hold simultaneously for all $t$ and $v$. On this event the maximum of
the lower bounds and the minimum of the upper bounds also bracket
$\Fbar_t(v)$. Since $\Fbar_t(v)\in[0,1]$, the band may also be intersected
with $[0,1]$.

\emph{Width.} Fix $t$ and $v$, and put $N:=N_t(v)$. Since
$\NR_t(v)+\NL_t(v)=N$, the band at $(\lambda_j,\alpha_j)$ has width
\[
  w_j:=\frac Nt+\frac{\lambda_j}4+\frac{2c_j}{\lambda_jt},
  \qquad c_j:=N\ln2+\ln\frac2{\alpha_j},
\]
and the combined band has width at most $\min_jw_j$. Put
\[
  c:=N\ln2+\ln\frac2\alpha+2\ln\Bigl(2+\tfrac12\log_2t\Bigr),
  \qquad \lambda^*:=\sqrt{8c/t}.
\]
If $\lambda^*\ge1$, then $\sqrt{c/t}\ge1/\sqrt8$, and the claim holds because
the width is at most $1$. Otherwise let $j\ge1$ be the smallest index with
$\lambda_j\le\lambda^*$. By minimality $\lambda_{j-1}>\lambda^*$, so
$\lambda_j\in(\lambda^*/2,\lambda^*]$. Moreover
$j\le1+\log_2(1/\lambda^*)=1+\frac12\log_2t-\frac12\log_2(8c)\le\frac12\log_2t$,
because $c>3\ln2$ gives $8c>16$. Hence
$\ln\frac2{\alpha_j}=\ln\frac2\alpha+\ln\bigl((j+1)(j+2)\bigr)
\le\ln\frac2\alpha+2\ln(j+2)$, so $c_j\le c$. Using
$\lambda_j\le\lambda^*$ in the second term and $\lambda_j>\lambda^*/2$ in the
third,
\[
  w_j\ \le\ \frac Nt+\frac{\lambda^*}4+\frac{4c}{\lambda^*t}
  \ =\ \frac Nt+\frac3{\sqrt2}\sqrt{\frac ct}.
\]
Finally $N/t\le\sqrt{N/t}$ because $N\le t$, and
$c=O\bigl(N_t(v)+\ln\ln t+\ln\frac1\alpha\bigr)$. This gives
\eqref{eq:rate}.
\end{proof}

\cdfbandvar*

Let $G_t(v):=\frac1t\sum_{s\le t}F_s(\inf C_s(v))\le\Fb_t(v)$. The proof
shows the lower bound for $G_t(v)$, and hence for $\Fb_t(v)$.

\begin{proof}[Proof of Theorem~\ref{thm:cdfbandvar}]
\emph{Lower bound.} Apply Proposition~\ref{prop:A} to $f$ in exact form,
$\gamma_s(a)=K(\lambda,p_s(a))$, with untilted weights
$\pi(a\sL)=\pi(a\sR)=\frac12$.
By Lemma~\ref{lem:families}(ii), the floor of a cell $[l,r)$ is
$p_s(a)=F_s(l)$, so the charges are
\[
  c(a\stay)=e^{-K(\lambda,F_s(l))},\quad
  c(a\sL)=\tfrac12e^{-K(\lambda,F_s(l))},\quad
  c(a\sR)=\tfrac12e^{-\lambda-K(\lambda,F_s(l))}.
\]
Moreover $\Pi_t(v)=N_t(v)\ln2$,
$\check\mu_t(v)=\widehat q_t(v)$ because $f_s(v)=1$ exactly at $\sR$-moves,
and $P_t(v)=G_t(v)$ because $p_s(v)=F_s(\inf C_s(v))$.
Corollary~\ref{cor:A}(ii) at level $\alpha/2$ gives the lower bound.

\emph{Upper bound.} Apply Proposition~\ref{prop:A} to $g$ in exact form with
$\pi(a\sL)=\pi(a\sR)=\frac12$.
Then $\Pi_t(v)=N_t(v)\ln2$, and
$\check\mu_t(v)=1-\Fh_t(v)-\NL_t(v)/t=1-\widehat q^{\,\prime}_t(v)$ because
$g_s(v)=1$ exactly at $\sL$-moves. Moreover
$P^g_t(v)\le\frac1t\sum_s(1-F_s(v))=1-\Fb_t(v)$.
Corollary~\ref{cor:A}(ii) at level $\alpha/2$ gives
\[
  1-\Fb_t(v)\ \ge\ \frac{e^{\lambda(1-\widehat q^{\,\prime}_t(v))-c_t(N_t(v))}-1}{e^\lambda-1},
\]
which is the upper bound.
\end{proof}

We now give some intuition for why the resulting band adapts to the variance.
\begin{remark}[Why the band adapts to the variance]\label{rem:varadapt}
Write $c_t(v):=\bigl(N_t(v)\ln2+\ln\frac2\alpha\bigr)/t$ and
$\bar\sigma_t^2(v):=G_t(v)\bigl(1-G_t(v)\bigr)$. On the event of
Theorem~\ref{thm:cdfbandvar}, the last display in the proof of
Corollary~\ref{cor:A}(ii), divided by $t$, gives
$\lambda\hat q_t(v)\le\Lambda_{G_t(v)}(\lambda)+c_t(v)$. Subtracting
$\lambda G_t(v)$ from both sides gives
\[
\lambda\bigl(\hat q_t(v)-G_t(v)\bigr)\ \le\ K\bigl(\lambda,G_t(v)\bigr)+c_t(v),
\]
with $K$ as in Lemma~\ref{lem:bern}. By Bennett's bound, applied to
$B-p\le1$ with $B\sim\mathrm{Bernoulli}(p)$,
\[
K(\lambda,p)\le p(1-p)(e^\lambda-1-\lambda)\le(e-2)\lambda^2p(1-p)
\]
for $\lambda\in(0,1]$. Hence, for every $\lambda\in(0,1]$,
\[
\hat q_t(v)-G_t(v)\ \le\ (e-2)\,\lambda\,\bar\sigma_t^2(v)+\frac{c_t(v)}{\lambda}.
\]
The first term is the variance $\bar\sigma_t^2(v)$ of an indicator with mean
$G_t(v)$, multiplied by $\lambda$. The second collects the confidence level
$\ln\frac2\alpha$ and the cost $N_t(v)\ln2$ of uniformity over thresholds,
divided by $\lambda$.

The choice of $\lambda$ trades one off against the other. Choosing
$\lambda=\min\{1,\sqrt{c_t(v)/((e-2)\bar\sigma_t^2(v))}\}$ gives
\[
\hat q_t(v)-G_t(v)\ \le\ 2\max\Bigl\{\sqrt{(e-2)\,\bar\sigma_t^2(v)\,c_t(v)},\ c_t(v)\Bigr\}.
\]
The best $\lambda$ depends on the unknown $G_t(v)$. As in Corollary~\ref{cor:mixture} and Remark~\ref{rem:grid}, the band is therefore computed on the grid $\lambda_j=2^{-j}$ with error probabilities $\alpha_j$. This loses at most a constant factor and replaces $\ln\frac2\alpha$ by
$\ln\frac2{\alpha_j}$ in $c_t(v)$.\\
The variance $\bar\sigma_t^2(v)$ is that of a single Bernoulli variable with the time-averaged mean $G_t(v)=\frac1t\sum_{s\le t}F_s(\inf C_s(v))$, not an
average of conditional variances. Since $p\mapsto p(1-p)$ is concave,
\[
  \bar\sigma_t^2(v)\ \ge\ \frac1t\sum_{s\le t}p_s(v)\bigl(1-p_s(v)\bigr),
  \qquad p_s(v)=F_s\bigl(\inf C_s(v)\bigr),
\]
with equality when the floors $p_s(v)$ do not change with $s$. For i.i.d.\ data, $\bar\sigma_t^2(v)$
is close to $F(v)(1-F(v))$ once the cell of $v$ is small. Under nonstationarity the gap can be large. If the floors alternate between $0$ and $1$, every conditional variance vanishes while
$\bar\sigma_t^2(v)=\frac14$. The band thus adapts to the spread of the averaged distribution at $v$,
which is small in its tails, without estimating it.
\end{remark}

\newpage

\section{IMPORTANCE-WEIGHTED CDF}\label{sec:iw}\label{app:iw}
% =====================================================================

Let $(W_s,X_s)$ be $\cF_s$-measurable with $W_s\ge0$ and
$\E_{s-1}W_s<\infty$. With $F^W_s$, $H^>_s$ and $H_s$ as in
Lemma~\ref{lem:families}, put
\[
  \Fh^W_t(v):=\frac1t\sum_{s\le t}W_s\ind\{X_s\le v\},\qquad
  \Fb^W_t(v):=\frac1t\sum_{s\le t}F^W_s(v),\qquad
  \overline W_t:=\frac1t\sum_{s\le t}W_s .
\]
Fix $\lambda\in(0,1)$. All bounds below come from Corollary~\ref{cor:Bpsi},
applied to $f_s(v)=W_s\ind\{X_s\le v\}$ or to $g_s(v)=W_s\ind\{X_s>v\}$, with
centerings in $[0,1]$, floor $0$ at divisions, $\pi_s(a\stay)=1$ without a
division, and \emph{split weights} $\pi_s(a\sL)+\pi_s(a\sR)\le1$ at a
division, which may depend on $W_s$. In this setting $\check f_s(v)$ is
$f_s(v)$ with the observations that split the cell of $v$ counted as missing.
Since $f_s(v)=0$ at a left split, $\check f_s(v)$ differs from $f_s(v)$ only
at right splits. We write $\check g_s$ and $\Pi^>_t$ for the same
quantities for $g$.

\begin{theorem}\label{thm:iw-lower}
Let $\alpha\in(0,1)$, let $\Ah$ be a centering, and fix split weights. With
probability at least $1-\alpha$, simultaneously for all $t\ge1$ and
$v\in\mathbb R$,
\[
  \Fb^W_t(v)\ \ge\ G^W_t(v)\ \ge\ \frac1t\sum_{s\le t}\psi_\lambda\bigl(\check f_s(v);\Ah_s(v)\bigr)-\frac{\Pi_t(v)+\ln\frac1\alpha}{\lambda t},
\]
where $G^W_t(v):=\frac1t\sum_{s\le t}F^W_s(\inf C_s(v))$.
\end{theorem}

\begin{proof}
The hypotheses of Corollary~\ref{cor:Bpsi} hold by
Lemmas~\ref{lem:threshold} and~\ref{lem:families}, and $\lambda\Ah<1$. The
cell minimum $W_s\ind\{X_s\le\inf C_s(a)\}$ vanishes on a division, so the
floor $0$ there is the cell minimum. By homogeneity, the envelope of
Proposition~\ref{prop:B} on the path of $v$ is
$W_s\ind\{X_s\le v\}-F^W_s(\inf C_s(v))$, so
$\Dt_t(a_{t+1}(v))=t(\Fh^W_t(v)-G^W_t(v))$. The proof of
Corollary~\ref{cor:Bpsi} therefore gives the second inequality. The first
holds because $\inf C_s(v)\le v$ and $F^W_s$ is nondecreasing.
\end{proof}

\begin{remark}[Choice of centering]\label{rem:centering}
Any $\Ah_s(v)\in[0,1]$ that is $\cF_{s-1}$-measurable and constant on the
cells of $\cC_s$ is a valid centering in Theorems~\ref{thm:iw-lower}
and~\ref{thm:iw-two}. Since $\psi_\lambda(x;a)=x-\eta(\lambda(x-a))/\lambda$,
the lower bound of Theorem~\ref{thm:iw-lower} equals
\[
  \frac1t\sum_{s\le t}\check f_s(v)
  -\frac1{\lambda t}\sum_{s\le t}\eta\bigl(\lambda(\check f_s(v)-\Ah_s(v))\bigr)
  -\frac{\Pi_t(v)+\ln\frac1\alpha}{\lambda t}.
\]
The centering enters only through the middle term, which is nonnegative. By
Lemma~\ref{lem:fan}, this term is at most
$\kappa(\lambda)\frac1t\sum_{s\le t}(\check f_s(v)-\Ah_s(v))^2$, where
$\lambda/2\le\kappa(\lambda)\le\lambda/(2(1-\lambda))$. 
The band is therefore tighter the better $\Ah_s(v)$ predicts $\check f_s(v)$. Natural choices are the following.

\emph{No centering.} The choice $\Ah\equiv0$ is always a centering. Since $\psi_\lambda(x;0)=\ln(1+\lambda x)/\lambda$, the lower bound of Theorem~\ref{thm:iw-lower} then reads
\[
  \frac1{\lambda t}\sum_{s\le t}\ln\bigl(1+\lambda\check f_s(v)\bigr)
  -\frac{\Pi_t(v)+\ln\frac1\alpha}{\lambda t}.
\]
\emph{Running mean.} Put $\Ah_1(v):=0$ and, for $s\ge2$,
\[
  \Ah_s(v):=\min\left\{1,\ \frac1{s-1}\sum_{r<s}W_r\ind\{X_r\le v\}\right\}.
\]
For $r<s$ the indicator $\ind\{X_r\le v\}$ is constant on every cell of
$\cC_s$, so this is a centering. It estimates $\Fb^W_{s-1}(v)$, and the
clipping keeps $\lambda\Ah_s(v)<1$.

\emph{Discounted running mean.} Under drift the running mean lags behind. For
a fixed $\rho\in(0,1)$, put $\Ah_1(v):=0$ and, for $s\ge2$,
\[
  \Ah_s(v):=\min\left\{1,\ \frac{\sum_{r<s}\rho^{s-1-r}W_r\ind\{X_r\le v\}}{\sum_{r<s}\rho^{s-1-r}}\right\}.
\]
This is again a centering, and it tracks the recent conditional means.

\emph{Prior guess.} If a deterministic or $\cF_{s-1}$-measurable guess $F_0$ of
the target CDF is available, for instance from a model or a pilot study, then
$\Ah_s(v):=\min\{1,F_0(\inf C_s(v))\}$ is a centering. It is evaluated at the
left endpoint of the cell, because a centering must be constant on cells.

\emph{Oracle.} The centering $A_s(v):=F^W_s(\inf C_s(v))$ cannot be computed,
but it gives the explicit band of Corollary~\ref{cor:iw-explicit}, in which
Fan's inequality bounds the centering cost.

For the upper bound of Theorem~\ref{thm:iw-two}, the same constructions apply
with $W_r\ind\{X_r>v\}$ in place of $W_r\ind\{X_r\le v\}$. Several centerings
can also be combined. Running the theorem with $K$ centerings at level
$\alpha/K$ each and taking the largest lower bound only replaces
$\ln\frac1\alpha$ by $\ln\frac K\alpha$.
\end{remark}

\begin{theorem}\label{thm:iw-two}
Assume in addition that $\E_{s-1}W_s=1$ for all $s$. Let $\Ah$ and $\Ah^>$ be
centerings, and fix split weights for $f$ and for $g$. With probability at
least $1-\alpha$, simultaneously for all $t$ and $v$,
\[
  \frac1t\sum_{s\le t}\psi_\lambda\bigl(\check f_s(v);\Ah_s(v)\bigr)-\frac{\Pi_t(v)+\ln\frac2\alpha}{\lambda t}
  \ \le\ \Fb^W_t(v)\ \le\
  1-\frac1t\sum_{s\le t}\psi_\lambda\bigl(\check g_s(v);\Ah^>_s(v)\bigr)+\frac{\Pi^>_t(v)+\ln\frac2\alpha}{\lambda t}.
\]
\end{theorem}

\begin{proof}
The lower bound is Theorem~\ref{thm:iw-lower} at level $\alpha/2$. For the
upper bound, apply Corollary~\ref{cor:Bpsi} in the same way to $g_s$, whose
conditional mean is $H^>_s$, at level $\alpha/2$. By
Lemma~\ref{lem:families}, its floor is the cell minimum
$W_s\ind\{X_s\ge\sup C_s(a)\}$, which is $0$ on a division. As in the
previous proof, this gives
\[
  G^{W,>}_t(v)\ \ge\ \frac1t\sum_{s\le t}\psi_\lambda\bigl(\check g_s(v);\Ah^>_s(v)\bigr)-\frac{\Pi^>_t(v)+\ln\frac2\alpha}{\lambda t},
  \qquad G^{W,>}_t(v):=\frac1t\sum_{s\le t}H_s(\sup C_s(v)).
\]
Since $\sup C_s(v)>v$, we have $H_s(\sup C_s(v))\le H^>_s(v)$, and since
$\E_{s-1}W_s=1$, we have $H^>_s(v)=1-F^W_s(v)$. So
$G^{W,>}_t(v)\le1-\Fb^W_t(v)$. A union bound completes the proof.
\end{proof}

\begin{remark}[Unweighted case]\label{rem:iw-unweighted}
For $W_s\equiv1$, the condition $\E_{s-1}W_s=1$ holds and $\Fb^W_t$ is the
running-average CDF $\Fb_t$ of Section~\ref{sec:cdf}. With the split weights
$\frac12,\frac12$, Theorem~\ref{thm:iw-two} then gives a further band for
$\Fb_t$. That is, with probability at least $1-\alpha$, for all $t$ and $v$,
\[
  \frac1t\sum_{s\le t}\psi_\lambda\bigl(\check f_s(v);\Ah_s(v)\bigr)-\rho_t(v)
  \ \le\ \Fb_t(v)\ \le\
  1-\frac1t\sum_{s\le t}\psi_\lambda\bigl(\check g_s(v);\Ah^>_s(v)\bigr)+\rho_t(v),
\]
where $\rho_t(v):=\bigl(N_t(v)\ln2+\ln\frac2\alpha\bigr)/(\lambda t)$,
$\check f_s(v)=\ind\{X_s\le v\}$ and $\check g_s(v)=\ind\{X_s>v\}$ if $v$ is
not split at time $s$, and $\check f_s(v)=\check g_s(v)=0$ otherwise. Its
width depends on how well the centering predicts $\check f_s(v)$
(Remark~\ref{rem:centering}).
\end{remark}
\begin{lemma}[{\citealp[proof of Lemma~4.1]{fan2015exponentialinequalitiesmartingalesapplications}}]\label{lem:fan}
For $\lambda\in[0,1)$ and $\xi\ge-1$,
$0\le\eta(\lambda\xi)\le\psi_E(\lambda)\,\xi^2$, where
$\psi_E(\lambda):=-\lambda-\ln(1-\lambda)$.
\end{lemma}
\begin{proof}
This is inequality~(4.11) in the proof of Lemma~4.1 of
\citet{fan2015exponentialinequalitiesmartingalesapplications}. We include
their argument for completeness. For $y>-1$,
\[
  \eta(y)=\int_0^y\frac{u}{1+u}\,du=y^2\,k(y),\qquad
  k(y):=\int_0^1\frac{w}{1+wy}\,dw,
\]
by the substitution $u=wy$. The function $k$ is nonnegative and nonincreasing
on $(-1,\infty)$. Hence, for $\xi\ge-1$ and $\lambda\in(0,1)$,
\[
  \eta(\lambda\xi)=\lambda^2\xi^2k(\lambda\xi)\le\lambda^2\xi^2k(-\lambda)=\xi^2\,\eta(-\lambda)=\xi^2\psi_E(\lambda).
\]
For $\lambda=0$ both sides vanish.
\end{proof}

\begin{corollary}\label{cor:iw-explicit}
Assume $\E_{s-1}W_s=1$, fix split weights as in Theorem~\ref{thm:iw-two}, and
let $\kappa(\lambda):=\psi_E(\lambda)/\lambda$. Put
\[
  \check F^W_t(v):=\frac1t\sum_{s\le t}\check f_s(v),\qquad
  \widehat V_t(v):=\sum_{s\le t}\check f_s(v)^2,
\]
and define $\check F^{W,>}_t(v)$ and $\widehat V^>_t(v)$ in the same way from
$\check g_s(v)$. With probability at least $1-\alpha$, simultaneously for all $t$
and $v$,
\begin{align*}
  \Fb^W_t(v)&\ge\frac1{1+\kappa(\lambda)}\Bigl[\check F^W_t(v)-\frac{\kappa(\lambda)\widehat V_t(v)}t-\frac{\Pi_t(v)+\ln\frac2\alpha}{\lambda t}\Bigr],\\
  \Fb^W_t(v)&\le1-\frac1{1+\kappa(\lambda)}\Bigl[\check F^{W,>}_t(v)-\frac{\kappa(\lambda)\widehat V^>_t(v)}t-\frac{\Pi^>_t(v)+\ln\frac2\alpha}{\lambda t}\Bigr].
\end{align*}
\end{corollary}

\begin{proof}
\emph{Lower bound.} Use the oracle centering $A_s(v):=F^W_s(\inf C_s(v))$.
It lies in $[0,1]$ because $F^W_s\le\E_{s-1}W_s=1$, it is
$\cF_{s-1}$-measurable, and it depends on $v$ only through the cell
$C_s(v)$, so it is a centering. With this choice,
\[
  G^W_t(v)=\frac1t\sum_{s\le t}A_s(v).
\]
Theorem~\ref{thm:iw-lower} at level $\alpha/2$ therefore gives, on an event of
probability at least $1-\alpha/2$, for all $t$ and $v$,
\begin{equation}\label{eq:iwx-start}
  \sum_{s\le t}A_s(v)\ \ge\ \sum_{s\le t}\psi_\lambda\bigl(\check f_s(v);A_s(v)\bigr)
  -\frac{\Pi_t(v)+\ln\frac2\alpha}{\lambda}.
\end{equation}

We now bound each term $\psi_\lambda(\check f_s(v);A_s(v))$ from below. By
definition of $\psi_\lambda$,
\[
  \psi_\lambda\bigl(\check f_s(v);A_s(v)\bigr)
  =\check f_s(v)-\frac1\lambda\,\eta(\lambda\xi_s),
  \qquad \xi_s:=\check f_s(v)-A_s(v).
\]
Since $\check f_s(v)\ge0$ and $A_s(v)\le1$, we have $\xi_s\ge-1$, so
Lemma~\ref{lem:fan} gives $\eta(\lambda\xi_s)\le\psi_E(\lambda)\xi_s^2$.
Moreover, since $\check f_s(v)$ and $A_s(v)$ are both nonnegative,
\[
  \xi_s^2=\check f_s(v)^2-2\check f_s(v)A_s(v)+A_s(v)^2
  \ \le\ \check f_s(v)^2+A_s(v)^2 .
\]
With $\kappa(\lambda)=\psi_E(\lambda)/\lambda$, this yields
\[
  \psi_\lambda\bigl(\check f_s(v);A_s(v)\bigr)
  \ \ge\ \check f_s(v)-\kappa(\lambda)\bigl(\check f_s(v)^2+A_s(v)^2\bigr).
\]
Summing over $s\le t$ and using the definitions of $\check F^W_t(v)$ and
$\widehat V_t(v)$,
\[
  \sum_{s\le t}\psi_\lambda\bigl(\check f_s(v);A_s(v)\bigr)
  \ \ge\ t\,\check F^W_t(v)-\kappa(\lambda)\widehat V_t(v)
  -\kappa(\lambda)\sum_{s\le t}A_s(v)^2 .
\]
Since $A_s(v)\in[0,1]$, we have $A_s(v)^2\le A_s(v)$, and hence
$\sum_{s\le t}A_s(v)^2\le\sum_{s\le t}A_s(v)=t\,G^W_t(v)$. Inserting both
bounds into \eqref{eq:iwx-start} gives
\[
  t\,G^W_t(v)\ \ge\ t\,\check F^W_t(v)-\kappa(\lambda)\widehat V_t(v)
  -\kappa(\lambda)\,t\,G^W_t(v)-\frac{\Pi_t(v)+\ln\frac2\alpha}{\lambda}.
\]
Moving the term
$\kappa(\lambda)\,t\,G^W_t(v)$ to the left and dividing by $t(1+\kappa(\lambda))>0$ gives
\[
  G^W_t(v)\ \ge\ \frac1{1+\kappa(\lambda)}\Bigl[\check F^W_t(v)
  -\frac{\kappa(\lambda)\widehat V_t(v)}t
  -\frac{\Pi_t(v)+\ln\frac2\alpha}{\lambda t}\Bigr].
\]
Finally, $\Fb^W_t(v)\ge G^W_t(v)$ by Theorem~\ref{thm:iw-lower}, which is the
lower bound.

\emph{Upper bound.} Argue in the same way for $g$, with the oracle centering
$A^>_s(v):=H_s(\sup C_s(v))=\inf_{C_s(v)}H^>_s\in[0,1]$, and use
$G^{W,>}_t(v)\le1-\Fb^W_t(v)$ from the proof of Theorem~\ref{thm:iw-two}. A
union bound completes the proof.
\end{proof}

\begin{proof}[Proof of Theorem~\ref{thm:wfv}]
Apply Corollary~\ref{cor:iw-explicit} with the split weights $\pi_s(a\sL)=\pi_s(a\sR)=\frac{1}{2}$, giving $\Pi_t(v)=N_t(v)\ln2$. This gives the band.
\end{proof}

\newpage

\section{RIGHT-CENSORED DATA}\label{sec:cens}\label{app:censored}
% =====================================================================

At round $s$, a lifetime $T_s>0$ and a censoring time $\cens_s>0$ arrive. We
observe $\Tt_s:=T_s\wedge\cens_s$ and $\delta_s:=\ind\{T_s\le\cens_s\}$, which
are $\cF_s$-measurable. We assume that $(T_s,\cens_s)$ has a regular
conditional distribution given $\cF_{s-1}$, and put
\[
  S_s(u):=\PP_{s-1}(T_s>u),\qquad
  \Psi_s(u):=\PP_{s-1}(T_s\ge u)=S_s(u^-),\qquad
  F_s:=1-S_s,
\]
\[
  \Kt_s(u):=\PP_{s-1}(\cens_s>u),\qquad
  K_s(u):=\PP_{s-1}(\cens_s\ge u)=\Kt_s(u^-).
\]
Fix a horizon $\tau$, let $\cV:=(0,\tau)$, and let
$\overline S_t(v):=\frac1t\sum_{s\le t}S_s(v)$ be the target.

\begin{assumption}[Censoring]\label{ass:cens}
Conditionally on $\cF_{s-1}$, the times $T_s$ and $\cens_s$ are independent.
The function $u\mapsto\Kt_s(u)$ is $\cF_{s-1}$-measurable and known, and
$K_s(\tau)>0$.
\end{assumption}

Proposition~\ref{prop:cens-nopos} removes the condition $K_s(\tau)>0$.

\begin{lemma}\label{lem:cens-moments}
Let $U$ be $\cF_{s-1}$-measurable, and let
$W_s:=\delta_s\ind\{\Tt_s<\tau\}/K_s(\Tt_s)$.
\begin{enumerate}
\item[(i)] If $U\in(0,\tau)$, then $\Kt_s(U)\ge K_s(\tau)>0$ and
$\E_{s-1}[\ind\{\Tt_s>U\}]=S_s(U)\Kt_s(U)$.
\item[(ii)] If $U\in[0,\tau)$, then $\E_{s-1}[W_s\ind\{\Tt_s\le U\}]=F_s(U)$.
\item[(iii)] If $U\in(0,\tau]$, then
$\E_{s-1}[\ind\{\Tt_s\ge U\}]=\Psi_s(U)K_s(U)$.
\end{enumerate}
\end{lemma}

\begin{proof}
Conditionally on $\cF_{s-1}$, the pair $(T_s,\cens_s)$ has product law by
Assumption~\ref{ass:cens}, and $U$ is a constant.

(i) We have $\Kt_s(U)\ge K_s(\tau)$ since $U<\tau$, and
$\ind\{\Tt_s>U\}=\ind\{T_s>U\}\ind\{\cens_s>U\}$. Part (iii) is the same with
$\ind\{\Tt_s\ge U\}=\ind\{T_s\ge U\}\ind\{\cens_s\ge U\}$.

(ii) Since $\Tt_s=T_s$ on $\{\delta_s=1\}$ and $U<\tau$,
$W_s\ind\{\Tt_s\le U\}=\ind\{T_s\le U\}\ind\{\cens_s\ge T_s\}/K_s(T_s)$. On
$\{T_s\le U\}$ we have $K_s(T_s)\ge K_s(\tau)>0$. Integrate out $\cens_s$
first, by Fubini.
\end{proof}

\paragraph{Refinement on $(0,\tau)$.}
Run Definition~\ref{def:cellsR} on $\Tt_s$, using only the values in
$(0,\tau)$ and with $-\infty,+\infty$ replaced by $0,\tau$. For
$v\in(0,\tau)$, the cell of $v$ at time $s$ is
$[l_s(v),r_s(v))\cap(0,\tau)$, where
\[
  l_s(v):=\max\bigl(\{\Tt_r:r<s,\ 0<\Tt_r\le v\}\cup\{0\}\bigr),\qquad
  r_s(v):=\min\bigl(\{\Tt_r:r<s,\ v<\Tt_r<\tau\}\cup\{\tau\}\bigr).
\]
A probe $v$ is split at time $s$ if and only if $\Tt_s$ falls strictly inside its cell. In particular, observations $\Tt_s\ge\tau$ split no probe. Lemma~\ref{lem:threshold} holds verbatim. In
particular, $\Tt_s$ lies in the interior of no cell of $\cC_{s+1}$, and
$\sup C_{s+1}(v)\in(v,\tau]$. The counters $N_t$, $\NR_t$, $\NL_t$ and the sets
$\cR_t$, $\cL_t$ are as before.
\begin{theorem}\label{thm:cens-upper}
Let $\lambda,\alpha\in(0,1)$, let $W_s:=\delta_s\ind\{\Tt_s<\tau\}/K_s(\Tt_s)$,
and let $\Ah$ be a centering. Let $\check f_s(v)$ be as in
Appendix~\ref{sec:iw} for the family $f_s(v)=W_s\ind\{\Tt_s\le v\}$, with split
weights $\frac12,\frac12$. With probability at least $1-\alpha$,
simultaneously for all $t\ge1$ and $v\in(0,\tau)$,
\[
  \overline S_t(v)\ \le\ 1-\frac1t\sum_{s\le t}\psi_\lambda\bigl(\check f_s(v);\Ah_s(v)\bigr)
  +\frac{N_t(v)\ln2+\ln\frac1\alpha}{\lambda t}.
\]
\end{theorem}

Other split weights can be used as in Remark~\ref{rem:weights}.

\begin{proof}
We follow the proof of Theorem~\ref{thm:iw-lower} for the family
$f_s(v)=W_s\ind\{\Tt_s\le v\}$ on $(0,\tau)$. Its conditional mean is $F_s$
on $[0,\tau)$ by Lemma~\ref{lem:cens-moments}(ii), and $F_s$ is
nondecreasing and right-continuous with $F_s(0)=0$. The
proof of Lemma~\ref{lem:families} applies verbatim with $F^W_s$ replaced by
$F_s$. On the leftmost cell $(0,r)$ the cell minimum is $0$, since $\Tt_s>0$,
so the floor there is $0=F_s(0)$. With
$G_t(v):=\frac1t\sum_{s\le t}F_s(\inf C_s(v))$, the argument gives
\[
  \Fb_t(v)\ \ge\ G_t(v)\ \ge\ \frac1t\sum_{s\le t}\psi_\lambda\bigl(\check f_s(v);\Ah_s(v)\bigr)
  -\frac{N_t(v)\ln2+\ln\frac1\alpha}{\lambda t}.
\]
Finally, $\overline S_t=1-\Fb_t$.
\end{proof}

\paragraph{The at-risk probe.}
For $v\in(0,\tau)$, put $f_s(v):=\ind\{\Tt_s>v\}/\Kt_s(v)$. Its values lie in
$[0,1/K_s(\tau)]$, and its conditional mean is $S_s$ by
Lemma~\ref{lem:cens-moments}(i). So
this probe estimates the target $S_s(v)$ directly. It is not homogeneous,
because the weight $1/\Kt_s(v)$ varies within cells. For a cell $C$ write
$r_C:=\sup C\in(0,\tau]$, and put $Q_s:=1/K_s(\Tt_s)$ on
$\{\Tt_s<\tau\}$.

\begin{lemma}\label{lem:atrisk}
The at-risk probe has the following properties.
\begin{enumerate}
\item[(i)] For $C\in\cC_{s+1}$, we have
$\sup_Cf_s=\ind\{\Tt_s\ge r_C\}/K_s(r_C)$.
\item[(ii)] For $C=C_s(a)$, the variable
$y_s(a):=\ind\{\Tt_s\ge r_C\}/K_s(r_C)$ is a floor in the sense of Proposition~\ref{prop:B}, with $\E_{s-1}[y_s(a)]=\Psi_s(r_C)=\inf_CS_s$.
\end{enumerate}
\end{lemma}

\begin{proof}
(i) Let $C$ have endpoints $l<r$. Then $\Tt_s$ is not in the interior of $C$. Either $\Tt_s\in(0,\tau)$, and then Lemma~\ref{lem:threshold}(iii)
applies, or $\Tt_s\ge\tau\ge r$. Hence
$\ind\{\Tt_s>u\}=\ind\{\Tt_s\ge r\}$ for all $u\in C$. If $\Tt_s\ge r$, both indicators equal $1$, and if $\Tt_s\le l$, both equal
$0$, since $u\ge l$.
The map $u\mapsto1/\Kt_s(u)$ is nondecreasing, with left limit $1/K_s(r)$ at
$r$, so the supremum over $C$ equals $\ind\{\Tt_s\ge r\}/K_s(r)$. Moreover
$K_s(r)\ge K_s(\tau)>0$.

(ii) The cell $C$ and its right endpoint $r_C$ are known at time $s-1$, so
Lemma~\ref{lem:cens-moments}(iii) gives
$\E_{s-1}[y_s(a)]=\Psi_s(r_C)$. The infimum of the nonincreasing $S_s$ over
$C$ is its left limit at $r_C$, which is $\Psi_s(r_C)$.
\end{proof}

\begin{theorem}\label{thm:cens-lower}
Let $\lambda,\alpha\in(0,1)$, and let $\Ah$ be a centering. For $v\in(0,\tau)$
put $r_s(v):=\sup C_s(v)$ and
\[
  Y_s(v):=\frac{\ind\{\Tt_s\ge r_s(v)\}}{K_s(r_s(v))},\qquad
  \check S_t(v):=\frac1t\sum_{s\le t}Y_s(v),\qquad
  S^\eta_t(v):=\sum_{s\le t}\eta\bigl(\lambda(Y_s(v)-\Ah_s(v))\bigr).
\]
With probability at least $1-\alpha$, simultaneously for all $t\ge1$ and
$v\in(0,\tau)$,
\[
  \overline S_t(v)\ \ge\ \check S_t(v)-\frac1{\lambda t}\Bigl[S^\eta_t(v)+N_t(v)\ln2+\ln\frac1\alpha\Bigr].
\]
\end{theorem}

As in Appendix~\ref{sec:iw}, any split weights with
$\pi(a\sL)+\pi(a\sR)\le1$ can be used instead of $\frac12,\frac12$, with
$\sum_{s\in\cS_t(v)}\ln(1/\pi_s(v))$ in place of $N_t(v)\ln2$
(Remark~\ref{rem:weights}).

\begin{proof}
\emph{Setup.} Apply Proposition~\ref{prop:B} to the at-risk probe. Fix a
cell $a$, let $r:=\sup C_s(a)$ and $Z_s(a):=\ind\{\Tt_s\ge r\}/K_s(r)$.
We take the floor $y_s(a):=Z_s(a)$ of Lemma~\ref{lem:atrisk}(ii), with
$q_s(a)=\Psi_s(r)$. It is larger than the cell minimum, which divides by
$\Kt_s(\inf C_s(a))$ instead. On a division $\Tt_s<r$, so $Z_s(a)=0$. By
Lemma~\ref{lem:atrisk}(i), the maxima are $f^\uparrow=Z_s(a)$ for $\stay$,
$Q_s$ for $\sL=[l,\Tt_s)$, and $0$ for $\sR=[\Tt_s,r)$. The weights are
$\pi(a\stay)=1$ without a division, and $\pi(a\sL)=\pi(a\sR)=\frac12$ with a
division.
The charges of Proposition~\ref{prop:B} are then
\[
  c(a\stay)=e^{-\eta(\lambda(Z_s-\Ah))},\qquad
  c(a\sL)=\tfrac12(1-\lambda\Ah)e^{-\lambda(Q_s-\Ah)},\qquad
  c(a\sR)=\tfrac12e^{-\eta(-\lambda\Ah)}.
\]
\emph{Path cost.} We have $Y_s(v)=Z_s(a_s(v))$ at $\stay$-steps and $Y_s(v)=0$
at divisions. By Corollary~\ref{cor:B}, the cost per step is $\eta(\lambda(Y_s(v)-\Ah_s(v)))$ at $\stay$-steps,
$\ln2+\lambda Q_s+\eta(-\lambda\Ah_s(v))$ at $\sL$-steps, and
$\ln2+\eta(-\lambda\Ah_s(v))$ at $\sR$-steps.
So $\ell_t(v)=S^\eta_t(v)+\lambda\sum_{s\in\cL_t(v)}Q_s+N_t(v)\ln2$.

\emph{Envelope.} By Proposition~\ref{prop:B} and Lemma~\ref{lem:atrisk}(i),
with $r_{s+1}(v)=\sup C_{s+1}(v)$,
\[
  \bar f_s(a_{s+1}(v))=\frac{\ind\{\Tt_s\ge r_{s+1}(v)\}}{K_s(r_{s+1}(v))}-\Psi_s(r_s(v)).
\]
At $\stay$- and $\sR$-steps, $r_{s+1}(v)=r_s(v)$ and the first term is $Y_s(v)$, which vanishes at an $\sR$-step since $\Tt_s<r_s(v)$. At $\sL$-steps,
$r_{s+1}(v)=\Tt_s$, the first term is $Q_s$, and $Y_s(v)=0$.
Hence
$\Dt_t(a_{t+1}(v))=t\check S_t(v)+\sum_{s\in\cL_t(v)}Q_s-\sum_{s\le t}\Psi_s(r_s(v))$.

\emph{Conclusion.} By Corollary~\ref{cor:masterband}, on an event of probability at
least $1-\alpha$, for all $t$ and $v$, the terms $\lambda\sum_{\cL_t(v)}Q_s$
cancel, and
\begin{equation}\label{eq:D33}
  \lambda\Bigl(t\check S_t(v)-\sum_{s\le t}\Psi_s(r_s(v))\Bigr)\le S^\eta_t(v)+N_t(v)\ln2+\ln\tfrac1\alpha .
\end{equation}
Since $r_s(v)>v$, we have $\Psi_s(r_s(v))\le\PP_{s-1}(T_s>v)=S_s(v)$. Hence
$\sum_s\Psi_s(r_s(v))\le t\overline S_t(v)$, which gives the claim at $v$
itself.
\end{proof}

\begin{theorem}\label{thm:cens-two}\label{thm:censband}
Under the assumptions of Theorems~\ref{thm:cens-lower}
and~\ref{thm:cens-upper}, each applied at level $\alpha/2$, both bounds hold
simultaneously with probability at least $1-\alpha$.
\end{theorem}

\begin{proof}
Union bound.
\end{proof}

\begin{corollary}\label{cor:cens-explicit}
Let $\kappa(\lambda):=\psi_E(\lambda)/\lambda$,
$\widehat V^Y_t(v):=\sum_{s\le t}Y_s(v)^2$,
$\check F^W_t(v):=\frac1t\sum_{s\le t}\check f_s(v)$ and
$\widehat V^W_t(v):=\sum_{s\le t}\check f_s(v)^2$, with $\check f_s$ as in
Theorem~\ref{thm:cens-upper}. With probability at least $1-\alpha$,
simultaneously for all $t$ and $v\in(0,\tau)$,
\begin{align*}
  \overline S_t(v)&\ge\frac1{1+\kappa(\lambda)}\left[\check S_t(v)-\frac{\kappa(\lambda)\widehat V^Y_t(v)}t-\frac{N_t(v)\ln2+\ln\frac2\alpha}{\lambda t}\right],\\
  \overline S_t(v)&\le1-\frac1{1+\kappa(\lambda)}\left[\check F^W_t(v)-\frac{\kappa(\lambda)\widehat V^W_t(v)}t-\frac{N_t(v)\ln2+\ln\frac2\alpha}{\lambda t}\right].
\end{align*}
\end{corollary}

\begin{proof}
\emph{Lower bound.} Use the oracle centering $A_s(v):=\Psi_s(r_s(v))$. By
Lemma~\ref{lem:atrisk}(ii) it equals $\inf_{C_s(v)}S_s$, so it lies in
$[0,1]$, is $\cF_{s-1}$-measurable, and is indexed by cells.
Inequality~\eqref{eq:D33} at level $\alpha/2$ holds for all $v$.

Put $\Gamma_t(v):=\sum_{s\le t}\Psi_s(r_s(v))=\sum_{s\le t}A_s(v)$ and
$\xi_s:=Y_s(v)-A_s(v)\ge-1$.
Lemma~\ref{lem:fan} gives
$\eta(\lambda\xi_s)\le\psi_E(\lambda)(Y_s(v)^2+A_s(v)^2)$, and
$A_s(v)^2\le A_s(v)$ since $A_s(v)\in[0,1]$.
Hence $S^\eta_t(v)\le\psi_E(\lambda)(\widehat V^Y_t(v)+\Gamma_t(v))$, and
\eqref{eq:D33} becomes
\[
  \frac{\Gamma_t(v)}t\ \ge\ \frac1{1+\kappa(\lambda)}\Bigl[\check S_t(v)-\frac{\kappa(\lambda)\widehat V^Y_t(v)}t-\frac{N_t(v)\ln2+\ln\frac2\alpha}{\lambda t}\Bigr].
\]
Since $r_s(v)>v$, we have $\Gamma_t(v)\le t\overline S_t(v)$.

\emph{Upper bound.} In the proof of Theorem~\ref{thm:cens-upper} at level
$\alpha/2$, use the oracle centering $A_s(v):=F_s(\inf C_s(v))\in[0,1]$, and
argue exactly as in the proof of Corollary~\ref{cor:iw-explicit}. This gives
\[
  \Fb_t(v)\ge G_t(v)\ge\frac1{1+\kappa(\lambda)}\left[\check F^W_t(v)-\frac{\kappa(\lambda)\widehat V^W_t(v)}t-\frac{N_t(v)\ln2+\ln\frac2\alpha}{\lambda t}\right],
\]
and $\overline S_t=1-\Fb_t$. A union bound completes the proof.
\end{proof}

\paragraph{Without positivity.}
The condition $K_s(\tau)>0$ in Assumption~\ref{ass:cens} can be removed.
Take $\tau:=+\infty$, use the convention $0/0:=0$, and let
\[
  h_s:=\inf\{u>0:\ \Kt_s(u)=0\}\in(0,\infty]
\]
be the time after which subject $s$ is censored almost surely. It is
$\cF_{s-1}$-measurable, and $\Kt_s(v)>0$ if and only if $v<h_s$, that is, if
subject $s$ is uncensored at time $v$ with positive probability. For
instance, if subject $s$ is censored at a fixed time $m_s$, so that
$\cens_s=m_s$, then $h_s=m_s$. For a probe $v$, let
$\cE_t(v):=\{s\le t:\ v<h_s\}$ be the subjects that are uncensored at time $v$
with positive probability, let $n_t(v):=|\cE_t(v)|$, and, for $n_t(v)\ge1$,
\[
  \overline S^{\circ}_t(v):=\frac1{n_t(v)}\sum_{s\in\cE_t(v)}S_s(v).
\]
Put $W_s:=\delta_s/K_s(\Tt_s)$. Let $\check g_s(v):=W_s\ind\{\Tt_s\le v\}$ if
$v$ is not split at time $s$ and $r_s(v)\le h_s$, and $\check g_s(v):=0$
otherwise.

\begin{proposition}[Bands without positivity]\label{prop:cens-nopos}
Let $\lambda,\alpha\in(0,1)$, and let $\Ah^\circ$ be a centering. Let $L_t(v)$
be the lower bound of Theorem~\ref{thm:cens-lower} at level $\alpha/2$,
computed with $\tau=+\infty$ and the convention $0/0:=0$. With probability at
least $1-\alpha$, simultaneously for all $t\ge1$ and all $v>0$ with
$n_t(v)\ge1$,
\[
  \frac{t}{n_t(v)}\,L_t(v)\ \le\ \overline S^\circ_t(v)\ \le\
  1-\frac1{n_t(v)}\left[\sum_{s\le t}\psi_\lambda\bigl(\check g_s(v);\Ah^\circ_s(v)\bigr)
  -\frac{N_t(v)\ln2+\ln\frac2\alpha}{\lambda}\right].
\]
On the same event, $L_t(v)\le\overline S_t(v)$.
\end{proposition}

\begin{proof}
\emph{Lower bound.} Run the proof of Theorem~\ref{thm:cens-lower} with
$\tau=+\infty$. If $K_s(r)=0$, then $\ind\{\Tt_s\ge r\}=0$ almost surely, so
the convention is consistent, Lemma~\ref{lem:atrisk}(i) holds almost surely,
and $Q_s=1/K_s(\Tt_s)$ is finite almost surely. The floor
$y_s(a)=\ind\{\Tt_s\ge r\}/K_s(r)$ now has conditional mean
$q_s(a)=\Psi_s(r)\ind\{K_s(r)>0\}$. Using this $q_s(a)$ in
Proposition~\ref{prop:B}, the proof gives \eqref{eq:D33} with
$\Psi_s(r_s(v))$ replaced by
$q_s(v):=\Psi_s(r_s(v))\ind\{K_s(r_s(v))>0\}$, that is,
$t\,L_t(v)\le\sum_{s\le t}q_s(v)$. If $K_s(r_s(v))>0$, then
$\Kt_s(v)\ge K_s(r_s(v))>0$ because $v<r_s(v)$, so $s\in\cE_t(v)$, and
$\Psi_s(r_s(v))\le S_s(v)$. Hence
$\sum_{s\le t}q_s(v)\le\sum_{s\in\cE_t(v)}S_s(v)=n_t(v)\,\overline S^\circ_t(v)$.
This sum is also at most $t\,\overline S_t(v)$.

\emph{Upper bound.} Apply Proposition~\ref{prop:B} on the threshold
refinement of the $\Tt_s$ to the family
$g_s(v):=W_s\ind\{\Tt_s\le v\}\ind\{v<h_s\}$, with split weights
$\frac12,\frac12$ and centering $\Ah^\circ$. On a cell $[l,r)$ of $\cC_s$,
take as floor the cell minimum
$y_s(a):=W_s\ind\{\Tt_s\le l\}\ind\{r\le h_s\}$. If $r\le h_s$ and
$\Tt_s=T_s\le l$, then $K_s(T_s)\ge\Kt_s(l)>0$ because $l<h_s$. Conditioning
on $T_s$ therefore gives $q_s(a)=\E_{s-1}[y_s(a)]=F_s(l)\ind\{r\le h_s\}$.
For $v\in[l,r)$ this is at most $F_s(v)\ind\{v<h_s\}=\E_{s-1}[g_s(v)]$, so
$y_s(a)$ is a floor. By Corollaries~\ref{cor:B} and~\ref{cor:masterband} at
level $\alpha/2$, $\lambda\Dt_t(a_{t+1}(v))\le\ell_t(v)+\ln\frac2\alpha$.
Writing out $\Dt_t$ and $\ell_t$, the terms $\lambda f^\uparrow_s(v)$ cancel,
and
\[
  \sum_{s\le t}q_s(v)\ \ge\ \sum_{s\le t}\psi_\lambda\bigl(y_s(v);\Ah^\circ_s(v)\bigr)
  -\frac{N_t(v)\ln2+\ln\frac2\alpha}\lambda,
\]
where $y_s(v)$ and $q_s(v)$ are the floor of the cell of $v$ at step $s$ and
its conditional mean. If $v$ is not split at step $s$, then
$\Tt_s\le l_s(v)$ if and only if $\Tt_s\le v$, so $y_s(v)=\check g_s(v)$. If
$v$ is split, then $\Tt_s>l_s(v)$, so $y_s(v)=0=\check g_s(v)$. Finally,
$\sum_{s\le t}q_s(v)\le\sum_{s\in\cE_t(v)}F_s(v)
=n_t(v)\bigl(1-\overline S^\circ_t(v)\bigr)$. A union bound over the two
sides completes the proof.
\end{proof}

\begin{remark}[Capped weights]\label{rem:cens-floor}
When subject $s$ is unlikely to be uncensored at the relevant time, that is,
when $K_s$ is small there, the weights
$Y_s(v)=\ind\{\Tt_s\ge r_s(v)\}/K_s(r_s(v))$ and $W_s=\delta_s/K_s(\Tt_s)$ can
be large. Fix $c\in(0,1]$ before seeing the data, and replace $K_s$ and
$\Kt_s$ everywhere by $\max\{K_s,c\}$ and $\max\{\Kt_s,c\}$. This caps every
weight at $1/c$, that is, the weights become
\[
  Y^c_s(v):=\min\{Y_s(v),1/c\},\qquad W^c_s:=\min\{W_s,1/c\}.
\]

\emph{The capped bands are valid.} The proofs use the weights only through the
conditional means of the floors, which must be at most the targets, and
capping only lowers these means. For the at-risk probe, the floor of a cell
with right endpoint $r$ has conditional mean
\[
  \frac{\Psi_s(r)K_s(r)}{\max\{K_s(r),c\}}
  =\Psi_s(r)\min\Bigl\{1,\frac{K_s(r)}c\Bigr\}
  \ \le\ \Psi_s(r)\ind\{K_s(r)>0\}.
\]
For the observed events, conditioning on $T_s$ gives
$\E_{s-1}[W^c_s\ind\{\Tt_s\le U\}]\le F_s(U)$. All results of this section
therefore hold with $Y^c_s$ and $W^c_s$ in place of $Y_s$ and $W_s$. Each choice of $c$ gives its own supermartingale, and hence a band that holds
on its own event of probability at least $1-\alpha$. The value of $c$ must
therefore be fixed before seeing the data. It may depend on $\lambda$. Several
values of $c$ can be combined by a union bound, as for the grid of $\lambda$
in Corollary~\ref{cor:mixture}.

\emph{Capping does not help Theorems~\ref{thm:cens-upper}
and~\ref{thm:cens-lower} or Proposition~\ref{prop:cens-nopos}.} The lower
bound of Theorem~\ref{thm:cens-lower} is
\[
  \frac1t\sum_{s\le t}\psi_\lambda\bigl(Y_s(v);\Ah_s(v)\bigr)
  -\frac{N_t(v)\ln2+\ln\frac1\alpha}{\lambda t},
\]
and the upper bound of Theorem~\ref{thm:cens-upper} is $1$ minus the same
expression with the event weights $\check f_s(v)$ in place of $Y_s(v)$.
Proposition~\ref{prop:cens-nopos} has the same structure, with $\check g_s(v)$
on the upper side. Since $\psi_\lambda(x;a)$ is increasing in $x$, capping
lowers every term of these sums. It therefore lowers the lower bounds and
raises the upper bounds, so with the same centering it never helps. 

\emph{Capping helps Corollary~\ref{cor:cens-explicit}.} With
$\phi(x):=x-\kappa(\lambda)x^2$, we have
$\check S_t(v)-\kappa(\lambda)\widehat V^Y_t(v)/t=\frac1t\sum_{s\le t}\phi(Y_s(v))$,
so the lower bound of Corollary~\ref{cor:cens-explicit} is
\[
  \frac1{1+\kappa(\lambda)}\left[\frac1t\sum_{s\le t}\phi\bigl(Y_s(v)\bigr)
  -\frac{N_t(v)\ln2+\ln\frac2\alpha}{\lambda t}\right],
\]
and its upper bound is $1$ minus the same expression with $\check f_s(v)$ in
place of $Y_s(v)$. On both sides, the
band is tighter the larger each term $\phi(\cdot)$ is. The function $\phi$
increases up to $x^*:=1/(2\kappa(\lambda))$ and decreases afterwards. A weight
above $x^*$ therefore contributes less than the weight $x^*$ would, and a
weight above $2x^*$ even contributes a negative amount. Capping at $x^*$
removes this loss. If $2\kappa(\lambda)\le1$, take $c:=2\kappa(\lambda)$, so
that $1/c=x^*$. Capping then leaves the weights below $x^*$ unchanged and
moves the larger ones down to $x^*$, where $\phi$ is largest. Otherwise take
$c:=1$. Every nonzero weight is at least $1>x^*$, since $K_s\le1$, and capping
moves it down towards $x^*$. In both cases capping does not decrease any term
$\phi(\cdot)$, so it tightens both the lower and the upper bound of
Corollary~\ref{cor:cens-explicit}.

\end{remark}

\newpage

\section{MULTIVARIATE CDF}\label{app:mcdf}
% =====================================================================
Let $X_s\in\mathbb R^d$ and write $u\le v$ componentwise, with
$f_s(v)=\ind\{X_s\le v\}$, $F_s(v)=\PP_{s-1}(X_s\le v)$, and $\Fh_t$, $\Fb_t$
as in Section~\ref{sec:mcdf}.

\paragraph{The idea.}
In one dimension, $X_s$ cuts the cell containing it at the point $X_s$, and
this makes $v\mapsto\ind\{X_s\le v\}$ constant on every new cell
(Definition~\ref{def:homog}). In $d$ dimensions,
\[
  f_s(v):=\ind\{X_s\le v\}=1\quad\Longleftrightarrow\quad v\in O_s,
  \qquad\text{where } O_s:=\{v\in\mathbb R^d:\ v\ge X_s\}.
\]
Here $O_s$ is the orthant above $X_s$. Homogeneity therefore requires each cell
to be cut along the boundary of $O_s$. That boundary consists of the
hyperplanes $\{v_j=X_{s,j}\}$, $j=1,\dots,d$. Cutting a box along all $d$
hyperplanes produces up to $2^d$ pieces. Two observations keep this cheap.
First, a box only needs to be cut if $f_s$ is not already constant on it, that
is, if $O_s$ meets the box without containing it. Second, a hyperplane
$\{v_j=X_{s,j}\}$ that does not pass through the interior of the box produces
an empty piece, which can be ignored, so only the coordinates in which
$X_{s,j}$ lies strictly inside the box create new pieces.
Figure~\ref{fig:boxes} illustrates the four possible situations for $d=2$.

\begin{figure}[t]
\centering
\begin{tikzpicture}[scale=0.8,
  lab/.style={font=\scriptsize},
  capt/.style={font=\footnotesize,align=center}]
  % (a) X_s inside the box: cut in both coordinates
  \begin{scope}[shift={(0,0)}]
    \fill[black!15] (1,1) rectangle (2,2);
    \draw[thick] (0,0) rectangle (2,2);
    \draw[dashed] (1,0) -- (1,2);
    \draw[dashed] (0,1) -- (2,1);
    \fill (1,1) circle (1.6pt) node[below left,lab] {$X_s$};
    \node[capt] at (1,-1.2) {(a) inside:\\ $4$ pieces, $m=2$};
  \end{scope}
  % (b) X_s below the box in coordinate 2: cut only in coordinate 1
  \begin{scope}[shift={(3.6,0)}]
    \fill[black!15] (1,0) rectangle (2,2);
    \draw[thick] (0,0) rectangle (2,2);
    \draw[dashed] (1,0) -- (1,2);
    \draw[dotted] (1,-0.5) -- (1,0);
    \fill (1,-0.5) circle (1.6pt) node[left,lab] {$X_s$};
    \node[capt] at (1,-1.2) {(b) below in $v_2$:\\ $2$ pieces, $m=1$};
  \end{scope}
  % (c) X_s above the box in coordinate 2: f_s = 0 on the box
  \begin{scope}[shift={(7.2,0)}]
    \draw[thick] (0,0) rectangle (2,2);
    \fill (1,2.5) circle (1.6pt) node[left,lab] {$X_s$};
    \node[capt] at (1,-1.2) {(c) above in $v_2$:\\ no cut, $f_s\equiv0$};
  \end{scope}
  % (d) X_s below the box in both coordinates: f_s = 1 on the box
  \begin{scope}[shift={(10.8,0)}]
    \fill[black!15] (0,0) rectangle (2,2);
    \draw[thick] (0,0) rectangle (2,2);
    \fill (-0.4,-0.4) circle (1.6pt) node[left,lab] {$X_s$};
    \node[capt] at (1,-1.2) {(d) below $l$:\\ no cut, $f_s\equiv1$};
  \end{scope}
\end{tikzpicture}
\caption{Cutting a box $B$ (solid) at an observation $X_s$ (dot) in $d=2$.
The shaded part is $B\cap O_s=\{v\in B:v\ge X_s\}$, where $f_s=1$. The box is
cut only when the shaded part is neither empty nor all of $B$, as in (a) and
(b), and only along the coordinates in which $X_s$ lies strictly inside $B$
(dashed).}
\label{fig:boxes}
\end{figure}
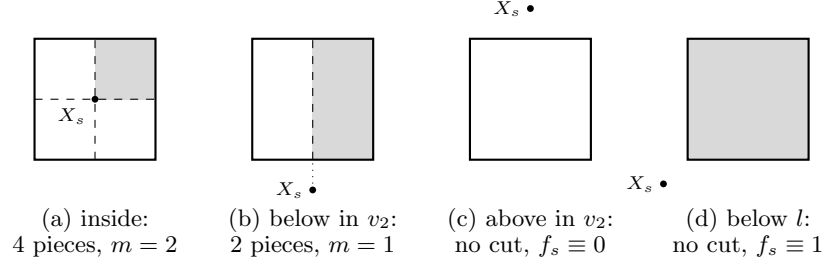

\begin{definition}[Box refinement]\label{def:boxes}
A \emph{half-open box} is a product $B=\prod_{j=1}^dI_j$, where each $I_j$ is
an interval $[l_j,r_j)$, with $l_j\in\mathbb R\cup\{-\infty\}$ and
$r_j\in\mathbb R\cup\{+\infty\}$. For $l_j=-\infty$ this means
$(-\infty,r_j)$. Its lower corner is $l=(l_1,\dots,l_d)$ and its upper corner
is $r=(r_1,\dots,r_d)$.

\emph{Pieces.} Given $x\in\mathbb R^d$ and a label
$\varepsilon\in\{0,1\}^d$, the $\varepsilon$-piece of $B$ at $x$ is
\[
  B^x_\varepsilon:=\prod_{j=1}^d I_j^{\varepsilon_j},\qquad
  I_j^{0}:=I_j\cap(-\infty,x_j),\quad I_j^{1}:=I_j\cap[x_j,\infty).
\]
The label says, coordinate by coordinate, whether a point lies to the left
($0$) or to the right ($1$) of $x$. The $2^d$ pieces partition $B$, and
$B^x_{(1,\dots,1)}=B\cap\{v\ge x\}$.

\emph{Refinement.} Use the alphabet $\Alp:=\{\stay\}\cup\{0,1\}^d$ and start from
$\cC_1:=\{\mathbb R^d\}$. At step $s$, a box $B\in\cC_s$ is \emph{cut} if and
only if
\[
  X_{s,j}<r_j\ \text{for all }j
  \qquad\text{and}\qquad
  X_{s,j}>l_j\ \text{for some }j .
\]
The first condition says that $O_s$ meets $B$, and the second says that $O_s$ does
not contain $B$. If $B$ is cut, it is replaced by its nonempty pieces
$B^{X_s}_\varepsilon$, and a probe in $B^{X_s}_\varepsilon$ receives the
symbol $\varepsilon$. If $B$ is not cut, it is kept, and its probes receive
$\stay$.
\end{definition}

Suppose $B$ is cut. In a coordinate with
$l_j<X_{s,j}<r_j$, both $I_j^0$ and $I_j^1$ are nonempty. In a coordinate
with $X_{s,j}\le l_j$, we have $I_j^0=\emptyset$ and $I_j^1=I_j$, so only
labels with $\varepsilon_j=1$ survive. The case $X_{s,j}\ge r_j$ is excluded
by the first cutting condition. Hence, with
\[
  m:=\#\{j:\ l_j<X_{s,j}<r_j\}\ \ge1,
\]
the number of coordinates in which $B$ is actually cut, exactly $2^m$ pieces
are nonempty. For $d=1$, a cut box has $m=1$ and the two pieces $[l,X_s)$ and
$[X_s,r)$, which is the refinement of Definition~\ref{def:cellsR}.

\begin{lemma}[Structure of the box refinement]\label{lem:boxes}
The box refinement has the following properties.
\begin{enumerate}
\item[(i)] Definition~\ref{def:boxes} defines a refinement scheme whose cells
are half-open boxes, and Assumption~\ref{ass:meas} holds.
\item[(ii)] The scheme is homogeneous for $f_s(v)=\ind\{X_s\le v\}$ and for
$1-f_s$.
On a box that is not cut, $f_s\equiv\ind\{X_s\le l\}$, where
$\ind\{X_s\le l\}:=0$ if some $l_j=-\infty$. On the pieces of a cut box,
$f_s\equiv1$ on $B^{X_s}_{(1,\dots,1)}$ and $f_s\equiv0$ on the other $2^m-1$
nonempty pieces.
\item[(iii)] Assumption~\ref{ass:floor} holds for $f_s$, with conditional
mean $F_s$, and for $1-f_s$, with conditional mean $1-F_s$. On a box $B$ with
corners $l$ and $r$, the cell minima are $\ind\{X_s\le l\}$ and
$1-\ind\{X_s<r\}$, where $X_s<r$ is meant componentwise.
\end{enumerate}
\end{lemma}

\begin{proof}
(i) By induction on $s$, the address $a_s(v)$ determines the box containing
$v$ and conversely. Each step either keeps the box, with symbol $\stay$, or
replaces it by the piece containing $v$, with symbol $\varepsilon$. The event
$\{a_s(v)=a\}$ is determined by finitely many comparisons among $v$ and
$X_{1:s-1}$, which gives (R3). Condition (R4) holds by construction.

(ii) On a piece $B^{X_s}_\varepsilon$, a point $v$ satisfies $v\ge X_s$ if
and only if $\varepsilon=(1,\dots,1)$, by the definition of the pieces. This
gives the values on a cut box. If $B$ is not cut, then either
$O_s\cap B=\emptyset$, so that $f_s\equiv0$ and $X_s\not\le l$, or
$B\subseteq O_s$, so that $f_s\equiv1$ and $X_s\le l$. The claims for $1-f_s$
follow.

(iii) For $v\in B$ we have $l\le v<r$. Hence
$\ind\{X_s\le l\}\le\ind\{X_s\le v\}$, with equality at $v=l$ when $l$ is
finite. If some $l_j=-\infty$, both sides vanish for $v_j$ small enough.
Moreover $\ind\{X_s\le v\}\le\ind\{X_s<r\}$, with equality at
$v=\max(X_s,l)$, taken componentwise, when $X_s<r$. So the cell minima are as
stated, they are at most $f_s(v)$ and $1-f_s(v)$ for every $v\in B$, and
Assumption~\ref{ass:floor} follows.
\end{proof}

For the path of a probe $v$, let $m_s(v)$ be the number $m$ of coordinates in
which the box of $v$ is cut at step $s$, for $s\in\cS_t(v)$. Let
\[
  M_t(v):=\sum_{s\in\cS_t(v)}m_s(v)
\]
be the total number of coordinate cuts of the box of $v$ up to time $t$.

For a probe $v$, let $N_t(v):=|\cS_t(v)|$ be the number of steps at which its
box is cut, and split these steps according to whether $v$ lands in the piece above $X_s$,
\[
  N^{\le}_t(v):=\#\{s\in\cS_t(v):\ X_s\le v\},\qquad
  N^{\not\le}_t(v):=N_t(v)-N^{\le}_t(v).
\]

\begin{proof}[Proof of Theorem~\ref{thm:mcdf}]
\emph{Lower bound.} Apply Corollary~\ref{cor:A}(i) to $f_s$ at level
$\alpha/2$, with $\gamma\equiv\lambda^2/8$ and uniform weights $2^{-m}$ on the
$2^m$ nonempty pieces of a cut box. Then
$\Pi_t(v)=\sum_{s\in\cS_t(v)}m_s(v)\ln2=M_t(v)\ln2$. At a cut, $f_s(v)=1$
exactly when $v$ lies in the piece above $X_s$, that is, when $X_s\le v$. So
$\check\mu_t(v)=\Fh_t(v)-N^{\le}_t(v)/t$.

\emph{Upper bound.} Apply the same argument to $1-f_s$, whose conditional mean
is $1-F_s$. Now $1-f_s(v)=1$ at exactly the cut steps with $X_s\not\le v$, so
$\check\mu^{1-f}_t(v)=1-\Fh_t(v)-N^{\not\le}_t(v)/t$. This gives
$1-\Fb_t(v)\ge1-\Fh_t(v)-N^{\not\le}_t(v)/t-\rho_t(v)$.

\emph{Invariance.} The cutting rule and the labels depend only on the
coordinatewise order of $v$ and $X_{1:t}$.
\end{proof}
\newpage

\section{DETAILS OF THE NUMERICAL EXPERIMENTS}\label{app:experiments}

\begin{proposition}[Monotone and cross-time closure]\label{prop:closure}
Let $(L_t,U_t)$ cover $\Fbar_t(v)$ for all $t,v$ on an event $E$. Define
$L^*_t(v):=\sup_{u\le v}L_t(u)$, $U^*_t(v):=\inf_{u\ge v}U_t(u)$, and recursively
$L^{**}_t:=\max\left\{L^*_t,\frac{t-1}tL^{**}_{t-1}\right\}$ and
$U^{**}_t:=\min\left\{U^*_t,\frac{(t-1)U^{**}_{t-1}+1}t\right\}$,
with $L^{**}_0:=0$ and $U^{**}_0:=1$. On $E$, these also cover $\Fbar_t(v)$ for
all $t,v$, and they are nested inside $(L_t,U_t)$.
\end{proposition}

\begin{proof}
Work on $E$. Since $\Fbar_t$ is nondecreasing, $L_t(u)\le\Fbar_t(u)\le\Fbar_t(v)$
for $u\le v$, so $L^*_t(v)\le\Fbar_t(v)$, and likewise $U^*_t(v)\ge\Fbar_t(v)$.
Since $t\Fbar_t=(t-1)\Fbar_{t-1}+F_t$ with $F_t\in[0,1]$,
$\frac{t-1}t\Fbar_{t-1}\le\Fbar_t\le\frac{(t-1)\Fbar_{t-1}+1}t$, and induction on
$t$ gives $L^{**}_t\le\Fbar_t\le U^{**}_t$. Finally $L^{**}_t\ge L^*_t\ge L_t$ and
$U^{**}_t\le U^*_t\le U_t$.
\end{proof}

The same argument applies to $\overline F^W_t$, whose terms
$F^W_s(v)=\mathbb E_{s-1}[W_s\mathbf 1\{X_s\le v\}]$ lie in $[0,1]$ and are
nondecreasing in $v$ when all $W_s$ are non-negative. It also applies to
$\overline S_t$, with the monotone part reversed:
$L^*_t(v)=\sup_{u\ge v}L_t(u)$ and $U^*_t(v)=\inf_{u\le v}U_t(u)$. The closure
costs no error probability and can only tighten a band. It matters for the
threshold bands, which are not monotone in $v$ because $N_t(v)$ is not.

\paragraph{Our bands.}
All bands use $\alpha=0.05$ and are combined over $\lambda_j=2^{-j}$ as in
Corollary~\ref{cor:mixture}. We take $\alpha_j=\alpha/((j+1)(j+2))$,
$j=0,\dots,23$, for Theorem~\ref{thm:cdfbandvar}, and $\alpha_j=\alpha/(j(j+1))$,
$j=1,\dots,16$, for Theorem~\ref{thm:iw-two} and Theorem~\ref{thm:cens-two}
, which need $\lambda<1$. Truncating the grid only lowers the total error
probability. Every band is closed with Proposition~\ref{prop:closure} at every
$t$. For Theorem~\ref{thm:cdfbandvar} we do this exactly on all cells; when a
cell splits, both children inherit the closure of their parent. The weighted and
censored bands are tracked on a fixed grid of probes.

\begin{remark}[Closure and the baseline]\label{rem:closure-mh}
Proposition~\ref{prop:closure} applies to any band that covers
$\overline F_t(v)$ for all $t$ and $v$, including the MH bands, but in
Section~\ref{sec:experiments} we apply it only to ours. The MH bands are
already monotone in $v$, and the cross-time closure changes them little,
because their grid is fixed in advance and their width varies smoothly
with~$t$. Our band, in contrast, is not monotone in $v$, and $N_t(v)$ jumps
when the cell of $v$ splits, so the closure matters for ours.
\end{remark}

\paragraph{Baseline.}
We compare with \citet{mineiro2023} (MH), using the authors' reference
implementation (\url{github.com/microsoft/csrobust}, MIT License) with error
$\alpha/2$ on each side. Their CDF bands evaluate a beta-binomial mixture
confidence sequence on a dyadic grid fixed in advance. For data in $[0,1]$
(Figure~2(a), (c), (d)) we use their algorithm for the unit interval with the
uniform reference measure, as in their experiments, and for importance weights
(Figure~3(a)) their DDRM confidence sequence with $k=0.25$. In both cases their
functions are used verbatim, and we only add a wrapper that evaluates them at
our probes. For Unif$(0,x_{\max})$ (Figure~2(b)) we use their variant for
arbitrary support (Algorithm~3 in their Appendix~D.1), which needs no reference
measure, with one change. Algorithm~3 bounds the CDF at $v$ on level $d$ through
the grid point next to $v$, $\lceil 2^d v\rceil/2^d$ for the upper bound and
$\lfloor 2^d v\rfloor/2^d$ for the lower bound, whose error budget decreases
with its index. Their implementation instead moves this index to that of the
nearest observation when the observation is closer to the origin
\citep[App.~D.1]{mineiro2023}, and skips the level when the new index is
even. For positive data, this only affects their lower bound, and for probes
just below a dyadic point it can make levels unusable. We therefore charge the
grid point next to $v$ at every level, exactly as in Algorithm~3. This changes
two lines of their code, which are marked in the supplementary material. Every
grid point keeps its pre-allocated error budget, so the band remains valid.

\subsection{CDF Bands}

Within each of Figure~\ref{fig:cdf}(a) and (b), all distributions use the same uniforms
$U_s$, $s\le 10^4$, with a different stream in each panel. The Beta$(6,b)$
samples are $F_b^{-1}(U_s)$ for $b\in\{3,3\cdot 10^4\}$. We plot MH for
$b=3$ and $b=3\cdot 10^4$. The
Unif$(0,x_{\max})$ samples are $x_{\max}U_s$. Both maps are strictly increasing, so our band at a given
quantile level is identical for every $b$ and $x_{\max}$. Independent
samples would change it only through sampling noise. For i.i.d.\ data from a
continuous distribution, our band at the quantile level $q$ depends only on the
relative order of the observations and the probe, and the law of this order
does not depend on the distribution. In
Figure~\ref{fig:cdf}(d) the data are i.i.d.\ Beta$(6,3)$ up to $t=10^5$.

The continuous P\'olya urn of MH draws
$X_s\sim\mathrm{Beta}(2+\gamma_s n^{>}_{s},\,2+\gamma_s n^{\le}_{s})$ with
$\gamma_s=s^{-1/2}$, where $n^{>}_s$ and $n^{\le}_s$ count the earlier
observations above $1/2$ and at most $1/2$. Figure~\ref{fig:cdf}(c) uses the
two seeds of MH, for which the urn concentrates near $0$ and near $1$, and shows
the bands at $t=10^3$.
MH pay for every point of their dyadic grid in a union bound, more for deeper
levels and, in the variant for arbitrary support, for points far from $0$. Data
concentrated on a short interval need deep levels, and data far from $0$ need
points far from $0$, which is why their width grows with $b$ in
Figure~\ref{fig:cdf}(a) and changes with $x_{\max}$ in Figure~\ref{fig:cdf}(b).

\subsection{Importance-Weighted CDF}

Figure~\ref{fig:ope}(a) uses the importance-weighted P\'olya urn of MH. The
weights are i.i.d.\ with $\mathbb P(W_s=4)=1/4$ and $\mathbb P(W_s=0)=3/4$, so
that $\mathbb EW_s=1$. The event $W_s=4$ means that the logging policy takes
the action of the target policy, which it does with probability $1/4$. Each
value $w\in\{0,4\}$ runs its own nonstationary P\'olya urn. Given $W_s=w$,
\[
X_s\sim\mathrm{Beta}\bigl(2+\gamma_s n^{>}_s(w),\;2+\gamma_s n^{\le}_s(w)\bigr),
\qquad \gamma_s=s^{-1/2},
\]
where $n^{>}_s(w)$ and $n^{\le}_s(w)$ count the $X_r$, $r<s$, with $W_r=w$ that
are above $1/2$ and at most $1/2$. Each urn reinforces only its own past, so
the two urns can concentrate at opposite ends of $[0,1]$. Let $F^{(w)}_s$ be
the conditional CDF of the urn with weight $w$. Then
$\mathbb E_{s-1}[W_s\mathbf 1\{X_s\le v\}]=F^{(4)}_s(v)$, so the target
$\overline F^W_t$ is the running average of the $F^{(4)}_s$, while the factual
running-average CDF in the figure mixes both urns. We use
Theorem~\ref{thm:iw-two}, centred at the clipped running means of
Remark~\ref{rem:centering},
$\widehat A_s(v)=\min\{1,(s-1)^{-1}\sum_{r<s}W_r\mathbf 1\{X_r\le v\}\}$ for
the lower bound and the same with $\mathbf 1\{X_r>v\}$ for the upper bound,
with $\widehat A_1:=0$. 

\subsection{Censored Data}

Subjects arrive one at a time and belong to one of two groups, which are not
observed. In the high-risk group, $Z_s=1$, lifetimes are exponential with mean
$0.7$, so $S_1(v)=e^{-v/0.7}$. In the low-risk group, $Z_s=0$, they are Weibull
with shape $1.5$ and scale $2$, so $S_0(v)=e^{-(v/2)^{1.5}}$. The high-risk
group has the larger hazard on $(0,\tau)$, and its median lifetime is $0.49$,
against $1.57$.

The share of high-risk subjects depends on the outcomes observed so far.
Subject $s$ is high-risk with probability $p_s=0.2+0.6\,q_s$, where $q_s$ is
the fraction of the last 100 subjects (of all earlier subjects if $s\le 100$)
whose event was observed by time $1$, that is, with $\delta_r=1$ and
$\widetilde T_r\le 1$, and $p_1=0.5$. Many early events thus bring in more
high-risk subjects, who produce more early events, so the lifetimes are not
i.i.d. Since $q_s$ uses only observed data, $p_s$ is
$\mathcal F_{s-1}$-measurable, and the conditional survival function of $T_s$
is $S_s=(1-p_s)S_0+p_sS_1$. The target is
$\overline S_t(v)=t^{-1}\sum_{s\le t}S_s(v)$.

Each lifetime is censored by an independent $\mathsf C_s\sim\mathrm{Exp}(r)$,
and we observe $\widetilde T_s=T_s\wedge\mathsf C_s$ and
$\delta_s=\mathbf 1\{T_s\le\mathsf C_s\}$. We set $r=0.5$, so that the lifetimes
of 40\% of the subjects are censored, and $\tau=3$, so that
$K_s(\tau)=e^{-1.5}\approx0.22$ in Assumption~\ref{ass:cens}. We use Theorem~\ref{thm:cens-two}
with the clipped running means of Remark~\ref{rem:centering} as centerings, for
both the at-risk term and the event term. For reference,
Figure~\ref{fig:ope}(b) also shows the band of Theorem~\ref{thm:cdfbandvar},
with the closure, computed from the observed times $\widetilde T_s$. It ignores
the censoring and therefore targets the survival function of $\widetilde T_s$,
not $\overline S_t$.

\end{document}

%% file: fig_tree_column.tex
\definecolor{trpath}{RGB}{204,85,0}%
\definecolor{trstay}{RGB}{140,140,140}%

\begin{tikzpicture}[
  x=1cm, y=1cm,
  nd/.style={draw, rounded corners=1.5pt, fill=white, align=center,
             inner sep=0.8pt, minimum height=0.68cm, font=\scriptsize},
  iv/.style={nd, minimum width=1.08cm},
  sg/.style={nd, minimum width=0.70cm, fill=gray!13},
  hotn/.style={draw=trpath, line width=1pt, fill=trpath!10},
  split/.style={semithick, black},
  stay/.style={densely dashed, trstay},
  hot/.style={line width=1.4pt, trpath},
  sym/.style={font=\tiny, fill=white, inner sep=0.6pt},
  ssym/.style={sym, text=trstay},
  hsym/.style={sym, text=trpath},
  chg/.style={font=\scriptsize, text=trpath, inner sep=0.5pt},
]
\def\dx{1.04}
\def\dy{1.1}
\newcommand{\ad}[2]{{\footnotesize$#2$}\\[-1.5pt]$#1$}

% ---- level 0: C_1
\node[iv] (r) at (3*\dx, 0)
  {{\footnotesize$\mathbb R$}\\[-1.5pt]$\varnothing$};

% ---- level 1: C_2
\node[iv] (L) at (1*\dx, -\dy)
  {\ad{\mathsf L}{(-\infty,3)}};
\node[iv,hotn] (R) at (5*\dx, -\dy)
  {\ad{\mathsf R}{[3,\infty)}};

% ---- level 2: C_3
\node[iv] (LL) at (0*\dx, -2*\dy)
  {\ad{\mathsf L\mathsf L}{(-\infty,1)}};
\node[iv] (LR) at (2*\dx, -2*\dy)
  {\ad{\mathsf L\mathsf R}{[1,3)}};
\node[iv,hotn] (Rs) at (5*\dx, -2*\dy)
  {\ad{\mathsf R\stay}{[3,\infty)}};

% ---- level 3: C_4
\node[iv] (LLs) at (0*\dx, -3*\dy)
  {\ad{\mathsf L\mathsf L\stay}{(-\infty,1)}};
\node[iv] (LRs) at (2*\dx, -3*\dy)
  {\ad{\mathsf L\mathsf R\stay}{[1,3)}};
\node[iv,hotn] (RsL) at (4*\dx, -3*\dy)
  {\ad{\mathsf R\stay\mathsf L}{[3,5)}};
\node[iv] (RsR) at (6*\dx, -3*\dy)
  {\ad{\mathsf R\stay\mathsf R}{[5,\infty)}};

% ---- edges
\draw[split] (r.south) -- (L.north)
  node[sym,pos=.5] {$\trsym L$};
\draw[hot] (r.south) -- (R.north)
  node[hsym,pos=.5] {$\trsym R$};

\draw[split] (L.south) -- (LL.north)
  node[sym,pos=.5] {$\trsym L$};
\draw[split] (L.south) -- (LR.north)
  node[sym,pos=.5] {$\trsym R$};

\draw[hot,dash pattern=on 2.5pt off 1.5pt]
  (R.south) -- (Rs.north)
  node[hsym,pos=.5] {$\trsym s$};

\draw[stay] (LL.south) -- (LLs.north)
  node[ssym,pos=.5] {$\trsym s$};
\draw[stay] (LR.south) -- (LRs.north)
  node[ssym,pos=.5] {$\trsym s$};

\draw[hot] (Rs.south) -- (RsL.north)
  node[hsym,pos=.5] {$\trsym L$};
\draw[split] (Rs.south) -- (RsR.north)
  node[sym,pos=.5] {$\trsym R$};

% ---- step labels
\foreach \s/\yy in {1/0.5, 2/1.5, 3/2.5}
  \node[font=\scriptsize, text=black!70, anchor=east]
    at (-0.57, -\yy*\dy) {$X_{\s}$};

% ---- bottom rank-order line
\def\ya{-3*\dy-0.66}

\foreach \n/\xx in {LLs/0, LRs/2, RsL/4, RsR/6}
  \draw[densely dotted, gray] (\n.south) -- (\xx*\dx, \ya);

\draw[trpath, line width=2.5pt, opacity=.35]
  (3*\dx, \ya) -- (5*\dx, \ya);

\draw[-{Latex[length=1.6mm]}]
  (-0.57, \ya) -- (6*\dx+0.6, \ya)
  node[font=\scriptsize, anchor=south east, inner sep=1pt]
  {$\mathbb R$};

\foreach \xx/\lab in {1/X_2{=}1, 3/X_1{=}3, 5/X_3{=}5}{
  \fill (\xx*\dx, \ya) circle (1.6pt);
  \node[font=\scriptsize, below=0.5pt]
    at (\xx*\dx, \ya) {$\lab$};
}

\fill[trpath] (4*\dx, \ya-0.07)
  -- ++(-0.09,-0.15) -- ++(0.18,0) -- cycle;
\node[font=\scriptsize, text=trpath, below=3.5pt]
  at (4*\dx, \ya) {$v{=}4$};

\end{tikzpicture}%

%% file: refs.bib
@inproceedings{
mineiro2023,
title={Time-uniform confidence bands for the {CDF} under nonstationarity},
author={Paul Mineiro and Steven R Howard},
booktitle={Thirty-seventh Conference on Neural Information Processing Systems},
year={2023},
url={https://openreview.net/forum?id=39cFjnRpYm}
}

@article{Savisurvey,
author = {Aaditya Ramdas and Peter Gr{\"u}nwald and Vladimir Vovk and Glenn Shafer},
title = {{Game-Theoretic Statistics and Safe Anytime-Valid Inference}},
volume = {38},
journal = {Statistical Science},
number = {4},
publisher = {Institute of Mathematical Statistics},
pages = {576 -- 601},
year = {2023},
doi = {10.1214/23-STS894},
URL = {https://doi.org/10.1214/23-STS894}
}

@article{10.1214/aoms/1177696786,
author = {Herbert Robbins},
title = {{Statistical Methods Related to the Law of the Iterated Logarithm}},
volume = {41},
journal = {The Annals of Mathematical Statistics},
number = {5},
publisher = {Institute of Mathematical Statistics},
pages = {1397 -- 1409},
year = {1970},
doi = {10.1214/aoms/1177696786},
URL = {https://doi.org/10.1214/aoms/1177696786}
}

@article{10.1214/20-AOS1991,
author = {Steven R. Howard and Aaditya Ramdas and Jon McAuliffe and Jasjeet Sekhon},
title = {{Time-uniform, nonparametric, nonasymptotic confidence sequences}},
volume = {49},
journal = {The Annals of Statistics},
number = {2},
publisher = {Institute of Mathematical Statistics},
pages = {1055 -- 1080},
year = {2021},
doi = {10.1214/20-AOS1991},
URL = {https://doi.org/10.1214/20-AOS1991}
}

@article{10.3150/21-BEJ1388,
author = {Steven R. Howard and Aaditya Ramdas},
title = {{Sequential estimation of quantiles with applications to A/B testing and best-arm identification}},
volume = {28},
journal = {Bernoulli},
number = {3},
publisher = {Bernoulli Society for Mathematical Statistics and Probability},
pages = {1704 -- 1728},
year = {2022},
doi = {10.3150/21-BEJ1388},
URL = {https://doi.org/10.3150/21-BEJ1388}
}

@inproceedings{
chandak2021universal,
title={Universal Off-Policy Evaluation},
author={Yash Chandak and Scott Niekum and Bruno Castro da Silva and Erik Learned-Miller and Emma Brunskill and Philip S. Thomas},
booktitle={Advances in Neural Information Processing Systems},
editor={A. Beygelzimer and Y. Dauphin and P. Liang and J. Wortman Vaughan},
year={2021},
url={https://openreview.net/forum?id=09D1DIEHrJO}
}

@inproceedings{
huang2021offpolicy,
title={Off-Policy Risk Assessment in Contextual Bandits},
author={Audrey Huang and Liu Leqi and Zachary Chase Lipton and Kamyar Azizzadenesheli},
booktitle={Advances in Neural Information Processing Systems},
editor={A. Beygelzimer and Y. Dauphin and P. Liang and J. Wortman Vaughan},
year={2021},
url={https://openreview.net/forum?id=u9RvlvaBC7}
}

@article{anytimeoffpolicy, author = {Waudby-Smith, Ian and Wu, Lili and Ramdas, Aaditya and Karampatziakis, Nikos and Mineiro, Paul}, title = {Anytime-valid off-policy Inference for Contextual Bandits}, year = {2024}, issue_date = {September 2024}, publisher = {Association for Computing Machinery}, address = {New York, NY, USA}, volume = {1}, number = {3}, url = {https://doi.org/10.1145/3643693}, doi = {10.1145/3643693}, journal = {ACM / IMS J. Data Sci.}, month = may, articleno = {10}, numpages = {42} }

@article{CLERICO2026110558,
title = {Uniform mean estimation for monotonic processes},
journal = {Statistics \& Probability Letters},
volume = {228},
pages = {110558},
year = {2026},
issn = {0167-7152},
doi = {https://doi.org/10.1016/j.spl.2025.110558},
url = {https://www.sciencedirect.com/science/article/pii/S0167715225002032},
author = {Eugenio Clerico and Hamish E. Flynn and Patrick Rebeschini}
}

@Article{Rakhlin2015,
author={Rakhlin, Alexander
and Sridharan, Karthik
and Tewari, Ambuj},
title={Sequential complexities and uniform martingale laws of large numbers},
journal={Probability Theory and Related Fields},
year={2015},
month={Feb},
day={01},
volume={161},
number={1},
pages={111-153},
issn={1432-2064},
doi={10.1007/s00440-013-0545-5},
url={https://doi.org/10.1007/s00440-013-0545-5}
}

@book{ville1939,
  author    = {Ville, Jean},
  title     = {\'Etude critique de la notion de collectif},
  publisher = {Gauthier-Villars}, address = {Paris}, year = {1939}
}

@article{Ham_2026,
   title={Design-Based Confidence Sequences: A General Approach to Risk Mitigation in Panel Experiments},
   ISSN={1526-5501},
   url={http://dx.doi.org/10.1287/mnsc.2024.08235},
   DOI={10.1287/mnsc.2024.08235},
   journal={Management Science},
   publisher={Institute for Operations Research and the Management Sciences (INFORMS)},
   author={Ham, Dae Woong and Lindon, Michael and Tingley, Martin and Bojinov, Iavor},
   year={2026},
   month=Aug }

@ARTICLE{KT,
  author={Krichevsky, R. and Trofimov, V.},
  journal={IEEE Transactions on Information Theory}, 
  title={The performance of universal encoding}, 
  year={1981},
  volume={27},
  number={2},
  pages={199-207},
  doi={10.1109/TIT.1981.1056331}}

@article{chernoff,
author = {Howard, Steven and Ramdas, Aaditya and McAuliffe, Jon and Sekhon, Jagmohan},
year = {2020},
month = {01},
pages = {257-317},
title = {Time-uniform {C}hernoff bounds via nonnegative supermartingales},
volume = {17},
journal = {Probability Surveys},
doi = {10.1214/18-PS321}
}

@misc{fan2015exponentialinequalitiesmartingalesapplications,
      title={Exponential inequalities for martingales with applications}, 
      author={Xiequan Fan and Ion Grama and Quansheng Liu},
      year={2015},
      eprint={1311.6273},
      archivePrefix={arXiv},
      primaryClass={math.PR},
      url={https://arxiv.org/abs/1311.6273}, 
}

@inproceedings{alon,
author = {Alon, Noga and Ben-Eliezer, Omri and Dagan, Yuval and Moran, Shay and Naor, Moni and Yogev, Eylon},
title = {Adversarial laws of large numbers and optimal regret in online classification},
year = {2021},
isbn = {9781450380539},
publisher = {Association for Computing Machinery},
address = {New York, NY, USA},
url = {https://doi.org/10.1145/3406325.3451041},
doi = {10.1145/3406325.3451041},
booktitle = {Proceedings of the 53rd Annual ACM SIGACT Symposium on Theory of Computing},
pages = {447–455},
numpages = {9},
location = {Virtual, Italy},
series = {STOC 2021}
}

@article{PITTEL1984461,
title = {On growing random binary trees},
journal = {Journal of Mathematical Analysis and Applications},
volume = {103},
number = {2},
pages = {461-480},
year = {1984},
issn = {0022-247X},
doi = {https://doi.org/10.1016/0022-247X(84)90141-0},
url = {https://www.sciencedirect.com/science/article/pii/0022247X84901410},
author = {Boris Pittel}
}

@article{Luc,
author = {Devroye, Luc},
title = {A note on the height of binary search trees},
year = {1986},
issue_date = {July 1986},
publisher = {Association for Computing Machinery},
address = {New York, NY, USA},
volume = {33},
number = {3},
issn = {0004-5411},
url = {https://doi.org/10.1145/5925.5930},
doi = {10.1145/5925.5930},
journal = {J. ACM},
month = may,
pages = {489–498},
numpages = {10}
}

@article{betting,
    author = {Waudby-Smith, Ian and Ramdas, Aaditya},
    title = {Estimating means of bounded random variables by betting},
    journal = {Journal of the Royal Statistical Society Series B: Statistical Methodology},
    volume = {86},
    number = {1},
    pages = {1-27},
    year = {2024},
    month = {02},
    issn = {1369-7412},
    doi = {10.1093/jrsssb/qkad009},
    url = {https://doi.org/10.1093/jrsssb/qkad009},
    eprint = {https://academic.oup.com/jrsssb/article-pdf/86/1/1/56961777/qkad009.pdf},
}

@misc{mineiro2022lowerconfidencesequencechanging,
      title={A lower confidence sequence for the changing mean of non-negative right heavy-tailed observations with bounded mean}, 
      author={Paul Mineiro},
      year={2022},
      eprint={2210.11133},
      archivePrefix={arXiv},
      primaryClass={stat.ML},
      url={https://arxiv.org/abs/2210.11133}, 
}

@article{dkw,
author = {A. Dvoretzky and J. Kiefer and J. Wolfowitz},
title = {{Asymptotic Minimax Character of the Sample Distribution Function and of the Classical Multinomial Estimator}},
volume = {27},
journal = {The Annals of Mathematical Statistics},
number = {3},
publisher = {Institute of Mathematical Statistics},
pages = {642 -- 669},
year = {1956},
doi = {10.1214/aoms/1177728174},
URL = {https://doi.org/10.1214/aoms/1177728174}
}

@article{massart,
author = {Massart, Pascal},
year = {1990},
month = {07},
pages = {},
title = {The Tight Constant in the {D}voretzky-{K}iefer-{W}olfowitz Inequality},
volume = {18},
journal = {The Annals of Probability},
doi = {10.1214/aop/1176990746}
}
